\documentclass{fdsOF} 
\usepackage{amsmath}
\usepackage{paralist}
\usepackage[misc]{ifsym}
\usepackage{epsfig} 
\usepackage{epstopdf} 
\usepackage[colorlinks=true]{hyperref}
\hypersetup{urlcolor=blue, citecolor=red}
\allowdisplaybreaks

\usepackage{latexsym}
\usepackage{amssymb}
\usepackage{amsthm}

\def\currentvolume{X}
 \def\currentissue{X}
  \def\currentyear{200X}
   \def\currentmonth{XX}
    \def\ppages{X-XX}
     \def\DOI{10.3934/fods.2026019 }
     
\newtheorem{theorem}{Theorem}[section]
\newtheorem{corollary}[theorem]{Corollary}

\newtheorem{lemma}[theorem]{Lemma}
\newtheorem{proposition}[theorem]{Proposition}

\newtheorem{problem}{Problem}
\newtheorem{definition}[theorem]{Definition}
\newtheorem{assumption}[theorem]{Assumption}
\newtheorem{remark}[theorem]{Remark}
\newtheorem{example}[theorem]{Example}

\renewcommand{\hbar}{\overline{h}}
\newcommand{\gbar}{\overline{g}}
\newcommand{\Mbar}{\overline{M}}
\newcommand{\Omegabar}{\overline{\Omega}}
\newcommand{\Fbar}{\overline{\mathcal{F}}}
\newcommand{\PP}{\mathbb{P}}
\newcommand{\Qhat}{\widehat{Q}}
\newcommand{\hhat}{\widehat{h}}

\title[Factorizable joint shift revisited]
{Factorizable joint shift revisited} 

\author[Dirk Tasche]{}


\begin{document}
\maketitle

\centerline{\scshape
Dirk Tasche$^{{\href{mailto:dirk.tasche@nwu.ac.za}{\textrm{\Letter}}}*1}$}

\medskip

{\footnotesize
 \centerline{$^1$Centre for Business Mathematics and Informatics, North-West University, South Africa}
} 

\bigskip

 \centerline{\ }

\begin{abstract}
Factorizable joint shift (FJS) represents a type of distribution shift (or dataset shift)
that comprises both covariate and label shift. Recently, it has been observed
that FJS actually arises from consecutive label and covariate (or vice versa) shifts.
Research into FJS so far has been confined mostly to the case of categorical labels.
We propose a framework for analysing distribution shift in the case of a general label space,
thus covering both classification and regression models. Based on the framework, we generalise
existing results on FJS to general label spaces and present and analyse a related extension to
label distribution estimation of the expectation maximisation (EM) algorithm for class prior probabilities. 
We also take a fresh
look at generalized label shift (GLS) in the case of a general label space.
\end{abstract}

\section{Introduction}

Classification and regression models may work very well on their respective training datasets 
but fail utterly when being deployed on test datasets. Such failure can be caused 
by distribution shift (also known as dataset shift) between the training and test datasets. 
For this reason, distribution shift and domain adaptation (a notion comprising techniques for
tackling distribution shift) has been a major research topic in machine learning for some time.
This paper takes the perspective of Kouw and Loog \cite{Kouw&Loog2019} and studies the case 
where feature observations from
the test dataset are available for analysis but observations of labels are missing.

Under these circumstances, without any assumptions on the nature of the distribution shift 
between the training and test datasets
meaningful prediction of the labels in the test dataset or of their distribution is not feasible.
See Kouw and Loog  \cite{Kouw&Loog2019} for a survey of approaches to domain adaptation and their related assumptions.
Arguably, covariate shift (also known as population drift)  and label shift 
(also known as prior probability shift or target shift) are the most popular specific distribution shift 
assumptions, both for their intuive appeal as
well as their computational manageability (Moreno-Torres et al.~\cite{MorenoTorres2012521}). 
However, exclusive covariate and label shift assumptions
have been criticised for being insufficient for common domain adaptation tasks 
(e.g.~Tachet des Combes et al.~\cite{tachetdescombes2020domain}).

He et al.~\cite{he2022domain} proposed the notion of factorizable joint shift (FJS) 
as an assumption for distribution shift
that covers both covariate and label shift. Recently, Dong et al.~\cite{Dong2025Factorizable} observed that FJS actually 
may be described as combination of label and covariate shifts, thus making the assumption more plausible.
Tasche~\cite{tasche2022factorizable} analysed FJS for categorical label spaces with regard to the potential and
limitations of domain adaptation based on this assumption. This paper updates the analyses of
Tasche~\cite{tasche2022factorizable} for FJS in the case of a general label space. More precisely, the main contributions
of this paper to the subject of distribution shift are the following:
\begin{itemize}
\item Introducing a framework for the description of non-specific distribution shift with general label spaces
(including Euclidean, categorical and mixed Eu\-clidean-categorical label spaces)
in order to facilitate the demarcation and analysis of specific types of distribution shift. 
\item Generalising the characterisation of FJS by Tasche~\cite{tasche2022factorizable} such
that both classification and regression problems are covered (Theorem~\ref{th:charac} below).
\item Presenting a `label distribution estimation' 
generalisation to general label spaces of the expectation maximisation (EM) algorithm  
for estimating target prior class probabilities by Saerens et al.~\cite{saerens2002adjusting} and showing that it
is fit for purpose (Theorem~\ref{th:alternateB} below).
\item Generalising a result by He et al.~\cite{he2022domain} on the relation between generalized label shift (GLS)
and FJS (Proposition~\ref{pr:GLS} below).
\end{itemize}

The setting and notation for this paper are set out in Section~\ref{se:shift}, with
Section~\ref{se:linking} presenting a framework for the description of general joint distribution shift and
some details being covered in Appendix~\ref{se:notation} and Appendix~\ref{se:facilitating}. 
Section~\ref{se:FJS} comprises analyses and results regarding FJS as well as
approaches to estimating the characterics of FJS, including the presentation and analysis of a general EM algorithm in
Section~\ref{se:probB} as well as an illustrative example of how the EM algorithm deals with
real-valued labels and categorical features in Section~\ref{se:EMfinite}. The relation
between GLS and FJS is analysed in Section~\ref{se:GLS}. Section~\ref{se:conclusions} concludes the paper,
with a summary and suggestions for further research. Appendix~\ref{se:DensEst} presents a description
of the classification approach to density estimation while Appendix~\ref{se:proofs} contains the proofs for some of
the results presented before.

\section{General joint distribution shift}
\label{se:shift}

We study distribution shift related to classification and regression problems in a probabilistic setting with feature 
(also known as~input, explanatory or independent) 
variables $X$ and label (also known as~output, response or dependent) 
variables $Y$. We assume that a joint source (also known as~training) 
probability distribution $P = P_{(X,Y)}$ 
of $X$ and $Y$ already has been learnt. Now we are presented with a joint target 
(also known as~test) distribution $Q = Q_{(X,Y)}$ 
of $X$ and $Y$ which is possibly different to $P$ and only partially observable. 
Lack of observations from
the target distribution $Q$ and more generally lack of information about $Q$ may become an issue in the presence
of distribution shift (Zhang et al.~\cite{zhang2023dive}, Section~4.9). The problem then is to infer information  about
the posterior (also known as~conditional) distribution $Q_{Y|X}$ of $Y$ given $X$ under the target distribution 
$Q$ from the
source distribution $P$ despite a lack of joint observations of $(X,Y)$ under $Q$. 
In Section~\ref{se:FJS} below, under an assumption of FJS we analyse the following two cases:
\begin{problem} \label{prB}  The target marginal distribution $Q_X$ of the feature variable is known but neither the marginal distribution
$Q_Y$ of the label variable nor any of the posterior distributions $Q_{X|Y}$ and $Q_{Y|X}$.
\end{problem}
\begin{problem}\label{prA} Both of the target marginal distributions $Q_X$ and $Q_Y$ of the feature and label variables 
respectively are known but
neither of the posterior distributions $Q_{X|Y}$ and $Q_{Y|X}$.
\end{problem}
Problem~\ref{prB} is an unsupervised domain adaptation problem where a classification or 
regression function is learnt on
a completely known dataset in a supervised manner and is then adapted to another dataset without label information
in an unsupervised manner (Kouw and Loog~\cite{Kouw&Loog2019}, Section~6.7). Problem~\ref{prA} where only joint 
observations of features and label are missing 
has less frequently been looked at in the machine learning literature. It is akin to copula estimation
(Durante and Sempi~\cite{durante2016principles})
 and as such has some practical relevance 
(see e.g.~Jaworski et al.~\cite{jaworski2013copulae}). Problem~\ref{prA} is investigated in this paper mainly
because its solutions can inform approaches to solutions for Problem~\ref{prB}.

In preparation for the discussion of FJS in Section~\ref{se:FJS}, in this section, Appendix~\ref{se:notation} and
Appendix~\ref{se:facilitating} we assume that
the source distribution $P$ and the target distribution $Q$ are both completely known.
This assumption is reasonable because
\begin{itemize}
\item it helps identifying characteristics of the distributions which can be estimated from 
observations under appropriate assumptions,
\item it helps identifying characteristics of the distributions for which reasonable assumptions can be made 
(based on exogenous information, on theoretical considerations etc.), and
\item knowledge of the characteristics can inform generative approaches to distribution shift 
when datasets need to be generated by simulation. 
\end{itemize}

\subsection{Setting}
\label{se:setting}

The setting for this paper is a standard setting for domain adaptation and therefore,
with the exception of some details of
the notation, similar to the setting described for example in the first paragraph of Section~2 of 
Azizzadenesheli~\cite{Azizzadenesheli2022}.
\begin{itemize}
\item Each instance (also known as~object) $\omega$ is an element of a set 
$\Omega$ which is called \emph{instance space}. 
Each instance $\omega$ has features which are represented collectively as $x \in \Omega_X$ and a label $y \in \Omega_Y$.
$\Omega_X$ and $\Omega_Y$ are called \emph{feature space} and \emph{label space} respectively.
We are interested in predicting\footnote{%
If $\Omega_Y$ is a finite or countably infinite set 
the prediction task is called \emph{classification}. Otherwise it is
called \emph{regression}.} $y$ from $x$ and assume that any information available for the prediction is captured by
$x$. Therefore, instances $\omega$ may be identified with combinations of $x$ and $y$ 
such that $\omega = (x,y)$. Consequently, without loss of generality it can be assumed that
$\Omega = \Omega_X \times \Omega_Y$.
\item We call the projection $X: \Omega \to \Omega_X, \omega \mapsto x = X(\omega)$  \emph{feature variable}
and the projection $Y: \Omega \to \Omega_Y, \omega \mapsto y = Y(\omega)$ \emph{label variable}.
\item Instances occur with different frequencies. For example, instances represented 
by a combination $\omega_0=(x_0, y_0)$ 
of features and label might be observed more often than instances with a combination $\omega_1=(x_1, y_1)$.
To reflect such different occurence frequencies and to together cover classification and regression as well as
to properly deal with mixed features consisting of both continuous and categorical components, 
the setting for this paper is probabilistic and formulated in terms of measure theory 
(cf.~Chapter~1 of Klenke~\cite{klenke2013probability}).
\item The feature space $\Omega_X$ and the label space $\Omega_Y$ are parts of measurable spaces 
$(\Omega_X, \mathcal{H})$ and $(\Omega_Y, \mathcal{G})$ respectively where the $\sigma$-algebra\footnote{%
Cf.~ Definition~1.2 of Klenke~\cite{klenke2013probability}.
Some authors prefer to talk about $\sigma$-algebras as \emph{information sets}, 
e.g.~Holzmann and Eulert~\cite{HolzmannInformation2014}. See Definition~1.76 of Klenke~\cite{klenke2013probability} for 
the notion of measurability.}  $\mathcal{H}$
is a collection of subsets of $\Omega_X$ and the $\sigma$-algebra $\mathcal{G}$
is a collection of subsets of $\Omega_Y$. The related $\sigma$-algebra $\mathcal{F}$ of subsets of the instance space 
$\Omega$ then is defined as the $\sigma$-algebra generated by the feature variable $X$ and the label variable $Y$, i.e.
\begin{equation}\label{eq:F}
\mathcal{F} \ = \ \sigma(X, Y)\ =\ \sigma\bigl(X^{-1}(\mathcal{H}) \cup Y^{-1}(\mathcal{G})\bigr) \ = \
\mathcal{H} \otimes \mathcal{G},
\end{equation}
with $X^{-1}(\mathcal{H}) = \{X^{-1}(H): H \in \mathcal{H}\}$ and 
$Y^{-1}(\mathcal{G}) = \{Y^{-1}(G): G \in \mathcal{G}\}$.
By definition of $\mathcal{F}$, $X$ is $\mathcal{H}$-$\mathcal{F}$-measurable and 
$Y$ is $\mathcal{G}$-$\mathcal{F}$-measurable.
\item The \emph{source distribution} $P = P_{(X,Y)}$ is a joint distribution of the feature variable $X$ and the
label variable $Y$ and as such defined on $(\Omega, \mathcal{F})$. In the following, the source distribution $P$ is
assumed to be completely known. The \emph{source feature distribution} $P_X$ is the image measure (or 
push-forward measure) of $P$ under $X$
and as such defined by $P_X[H] = P[X \in H]$, $H \in \mathcal{H}$. The \emph{source label distribution} $P_Y$
is the image measure of $P$ under $Y$
and as such defined by $P_Y[G] = P[X \in G]$, $G \in \mathcal{G}$. Both $P_X$ and $P_Y$ are unconditional distributions
of $X$ and $Y$ respectively 
in the sense that they do not take into account already observed realisations of the other variable. 
In the machine learning literature $P_X$ and $P_Y$ also are referred to as \emph{target prior distribution}.
\item The \emph{target distribution} $Q = Q_{(X,Y)}$ is another joint distribution of the feature variable $X$ and the
label variable $Y$ and as such also defined on $(\Omega, \mathcal{F})$. As mentioned above, in contrast to the source 
distribution, the target distribution $Q$ is
assumed to be only partially known to an extent to be specified below in Section~\ref{se:FJS}. 
The \emph{target feature distribution} $Q_X$ is the image measure of $Q$ under $X$
and as such defined by $Q_X[H] = Q[X \in H]$, $H \in \mathcal{H}$. The \emph{target label distribution} $Q_Y$
is the image measure of $Q$ under $Y$
and as such defined by $Q_Y[G] = Q[X \in G]$, $G \in \mathcal{G}$. $Q_X$ and $Q_Y$ are called source
prior distribution.
\item The source distribution $P$ and the target distribution $Q$ are different, i.e.\ $P \neq Q$.
In other words, there is a \emph{distribution shift} (see e.g.~Zhang et al.~\cite{zhang2023dive} 
between source and target.
Occasionally, non-specific distribution shift also is  called \emph{dataset shift} in the literature 
(Qui{\~n}onero-Candela et~al.~\cite{quinonero2009dataset}, Moreno-Torres et al.~\cite{MorenoTorres2012521}).
\end{itemize}

In the following, some notation and technical terms (including for conditional and unconditional expected values) 
are used which are broadly the same as the notation  and terms
used in textbooks on probability theory like Bauer~\cite{Bauer1981ProbTheory} or
Klenke~\cite{klenke2013probability}. See Appendix~\ref{se:notation} for a selection of terms and concepts that
are important for this paper.

\subsection{Linking source and target distributions in the face of distribution shift}
\label{se:linking}

Zhang et al.~\cite{Zhang:2013:TargetShift} commented on the problem of how
to deal with distribution shift: ``If the data distribution changes arbitrarily, training
data would be of no use to make predictions on the test
domain. To perform domain adaptation successfully,
relevant knowledge in the training (or source) domain
should be transferred to the test (or target) domain.'' 

In this paper, knowledge transfer from source distribution $P$ to target distribution $Q$ is 
expressed through a density, as stated in the following main assumption of the paper.
\begin{subequations}
\begin{assumption}[Absolute continuity of target distribution]\label{as:main}
In the setting of Section~\ref{se:setting},  $Q$ is absolutely continuous with respect to $P$, i.e.
\begin{equation}\label{eq:abscont}
P[F] \ = \ 0 \ \text{for}\ F\in \mathcal{F} \quad \Rightarrow \quad Q[F] \ = \ 0.
\end{equation}
$f = f(X,Y) = \frac{d Q}{d P}$ denotes the Radon-Nikodym derivative\footnote{%
The existence of $f$ follows from the Radon-Nikodym theorem (Corollary~7.34 of Klenke~\cite{klenke2013probability}).}
of $Q$ with respect to $P$, i.e.\ $f$ is a density of $Q$ with respect to $P$ in the sense that
$f$ is $\mathcal{F}$-Borel measurable and non-negative and it holds that
\begin{equation}\label{eq:transform}
Q[F] \ =\ E_P[f\,\mathbf{1}_F]\quad \text{for all}\ F \in \mathcal{F}.
\end{equation}
\end{assumption}
\end{subequations}

In the literature the `transfer' of information from $P$ to $Q$ by means of \eqref{eq:transform} is sometimes
called `importance reweighting' and the density $f$ is referred to
as `importance weights' (e.g.~Zhang et al.~\cite{Zhang:2013:TargetShift}, Kouw and Loog~\cite{Kouw&Loog2019}). 
More often than not, in the literature the reference measure for the density $f$ remains
undisclosed or appears to be the Lebesgue measure (e.g.~Zhang et al.~\cite{Zhang:2013:TargetShift}, 
Lipton et al.~\cite{pmlr-v80-lipton18a}).
Reid and Williamson~\cite{reid2011information}  and  Azizzadenesheli~\cite{Azizzadenesheli2022} are notable exceptions.

Observe that condition \eqref{eq:abscont} is equivalent to stating that
$Q[F] > 0$ for $F\in \mathcal{F}$ implies $P[F] > 0$.
Hence condition \eqref{eq:abscont} may be interpreted as 
`the source distribution $P$ is richer with information than the target distribution $Q$'.
Actually, this is a rather strong assumption such that there is a host of literature on approaches
to domain adaptation without Assumption~\ref{as:main}, including the seminal paper by 
Ben-David et al.~\cite{Ben-David2007Representations}. See also Johansson et al.~\cite{pmlr-v89-johansson19a} 
for analyses of the consequences of renouncing Assumption~\ref{as:main}.

For the rest of the paper, we suppose that Assumption~\ref{as:main} is true. If there are independent
joint feature-label observations from both the source and the target distributions, the density $f$
from the assumption can be estimated. Besides giving a short overview of the
available estimation methods, Filipovi{\'c} and Schneider~\cite{filipovic2025kernel} present methods based on
reproducing kernel Hilbert spaces (RKHS) which are applicable under mild conditions. An approach
based on probabilistic binary classification was proposed by Qin~\cite{Qin1998DensClass}. This approach is
attractive because it allows for measuring the similarity of the two distributions in question
with tools like the area under the curve (AUC) or the test statistic of the Kolmogorov-Smirnov test. 
For the readers' convenience, a description of the classification approach is given in 
Appendix~\ref{se:DensEst} below.

Proposition~\ref{pr:genProperties} records some immediate consequences of Assumption~\ref{as:main}.

\begin{proposition}\label{pr:genProperties}
Let $T = T(X,Y): \Omega \to \mathbb{R}$ be $\mathcal{F}$-Borel-measurable with $E_P[|T|] < \infty$ or
$T \ge 0$. Then Assumption~\ref{as:main} implies the following statements:
\begin{itemize}
\item[(i)] $Q_X$ has a density $h$ with respect to $P_X$ which can be represented as
\begin{align*}
h(x) & \ = \ E_P[f \,|\,X=x], \quad P_X\text{-a.s.\ for}\ x \in \Omega_X,\\
\intertext{and it holds that}
E_Q[T\,|\,X=x] & \ = \ \frac{E_P[T\,f\,|\,X=x]}{h(x)}, 
\end{align*}
for $x\in \Omega_X$ with $h(x) > 0$, i.e.\ with probability~$1$ under $Q_X$.
\item[(ii)]  $Q_Y$ has a density $g$ with respect to $P_Y$ which can be represented as
\begin{align*}
g(y) & \ = \ E_P[f \,|\,Y=y], \quad P_Y\text{-a.s.\ for}\ y \in \Omega_Y,\\
\intertext{and it holds that}
E_Q[T\,|\,Y=y] & \ = \ \frac{E_P[T\,f\,|\,Y=y]}{g(y)},
\end{align*}
for $y\in \Omega_Y$ with $g(y) > 0$, i.e.\ with probability~$1$ under $Q_Y$.
\end{itemize} 
\end{proposition}

\begin{proof}[Proof of Proposition~\ref{pr:genProperties}.] 
The statements about the densities follow from \eqref{eq:identity} and Lemma~\ref{le:density} in 
Appendix~\ref{se:notation} below.\\
The statements about the conditional expected values follow from the generalized Bayes formula, 
as stated in Lemma~A1 of Tasche~\cite{tasche2022factorizable}, the formulae for the densities and \eqref{eq:identity}.
\end{proof}

Since application of Fubini's theorem \eqref{eq:Fubini} cannot be justified 
in the absence of any assumptions on the properties of the source distribution $P$ in the setting as
specified in Section~\ref{se:setting}, it is not possible to draw further general conclusions about 
the composition of the density $f$ or the relationship between the posterior distributions of
$X$ given $Y$ under the target distribution $Q$ and the source distribution $P$ respectively. This is in contrast to
the situation studied by Tasche~\cite{tasche2022factorizable}. In Tasche~\cite{tasche2022factorizable}, thanks to the assumption that
the label variable is categorical, more refined formulae for $f$ (Theorem~1) and the posterior probabilities 
(Corollary~2, a more general version of Eq.~(2.4) of Saerens et al.~\cite{saerens2002adjusting}, but of similar shape) 
were derived.

Appendix~\ref{se:source} and Appendix~\ref{se:further} present additional assumptions on the properties of the 
source distribution $P$. 
Based on those assumptions,  some follow-ups of the statements in Proposition~\ref{pr:genProperties} are presented 
in Appendix~\ref{se:facilitating}
in order to facilitate numerical evaluations of the statistics
related to the target distribution $Q$ among other things.

\section{Factorizable joint shift (FJS) for general label spaces}
\label{se:FJS}

He et al.~\cite{he2022domain} proposed FJS  ``to handle the co-existence of sampling bias in covariates and
labels'', and consequently as a notion generalising both covariate and label shift.
Tasche~\cite{tasche2022factorizable} (Remark~1) characterised FJS in the case of categorical label variables 
as invariance of ratios of the feature 
densities conditional on different labels up to constant factors between the source and the target distribution.
In addition, he noted that in domain adaptation when only the feature marginal distribution of the target 
distribution is known (i.e.\ Problem~\ref{prB} of Section~\ref{se:shift}), 
assuming FJS without constraints between the source and the target distributions does not uniquely
determine the distribution shift. More recently, Dong et al.~\cite{Dong2025Factorizable} 
observed that FJS may be described
as combination of a label and a covariate shift. This section generalises these observations for the case of
FJS in the face of a general label space. In addition, solutions to Problems \ref{prA} and \ref{prB} of Section~\ref{se:shift}
under the assumption of FJS are discussed.

\subsection{Characterising FJS}
\label{se:characFJS}

\begin{definition}\label{de:FJS} In the setting of Section~\ref{se:setting},
$P$ and $Q$ are related through \emph{factorizable joint shift} (FJS) 
if there are an $\mathcal{H}$-measurable function $\hbar: \Omega_{X} \to [0,\infty)$ and a $\mathcal{G}$-measurable 
function $\gbar: \Omega_{Y} \to [0,\infty)$ such
that $\hbar(X)\,\gbar(Y)$ is a density of $Q$ with respect to $P$.
\end{definition}

If $P$ and $Q$ are related through FJS then in particular Asssumption~\ref{as:main} is satisfied.
Note that $\hbar$ and $\gbar$ in Definition~\ref{de:FJS} are unique only up to a constant factor $c > 0$, in 
the sense that 
\begin{equation}\label{eq:hbarc}
\frac{d Q}{d P} \quad = \quad \hbar_c(X)\,\gbar_c(Y),
\end{equation}
for $\hbar_c = c\,\hbar$ and $\gbar_c = \gbar / c$. As mentioned above, FJS encompasses both covariate shift
and label shift as special cases.

\begin{definition}\label{de:covariatelabelShift}
In the setting of Section~\ref{se:setting}, the notions of covariate and label shift respectively are
defined as follows:
\begin{itemize}
\item[(i)] $Q$ and $P$ are related through \emph{covariate shift} 
if for each $G \in \mathcal{G}$ 
there is an $\mathcal{H}$-Borel measurable function $\eta_G: \Omega_X \to [0,1]$ such that 
$Q[Y \in G\,|\,X=x] = \eta_G(x)$ $Q_X$-a.s.\ and $P[Y \in G\,|\,X=x] = \eta_G(x)$ $P_X$-a.s.\ hold.
\item[(ii)] $Q$ and $P$ are related through \emph{label shift}  if for each $H \in \mathcal{H}$ 
there is a $\mathcal{G}$-Borel measurable function $\gamma_H: \Omega_y\to [0,1]$ such that 
$Q[X \in H\,|\,Y=y] = \gamma_H(y)$ $Q_Y$-a.s.\ and $P[X \in H\,|\,Y=y] = \gamma_H(y)$ $P_Y$-a.s.\ hold.
\end{itemize}
\end{definition}

\begin{proposition}\label{pr:SpecCases}
In the setting of Section~\ref{se:setting} and under Assumption~\ref{as:main},  both
covariate shift and label shift in the sense of Definition~\ref{de:covariatelabelShift} imply 
FJS in the sense of Definition~\ref{de:FJS}  between the source distribution $P$ and the target distribution $Q$.\\
Conversely, if $P$ and $Q$ are related through FJS with $\gbar =1$ then $P$ and $Q$ are related also through
covariate shift. If $P$ and $Q$ are related through FJS with $\hbar =1$ then $P$ and $Q$ are related also through
label shift.
\end{proposition}

See Appendix~\ref{se:proofs} for a proof of Proposition~\ref{pr:SpecCases}. 
By the density chain rule, a sequence of FJS results in another FJS. In particular, by
Proposition~\ref{pr:SpecCases} any sequence of covariate and label shifts gives rise to an FJS.
Conversely, any FJS can be interpreted as consecutive label and covariate shifts or consecutive covariate and 
label shifts. The following proposition provides a streamlined version of Theorem~1 of 
Dong et al.~\cite{Dong2025Factorizable} in the more general setting of Section~\ref{se:setting}.
\begin{proposition}\label{pr:combined}
Let $P$ and $Q$ be related through FJS in the sense of Definition~\ref{de:FJS}. Define probability
measures $Q_L$ and $Q_C$ on $(\Omega, \mathcal{F})$ by
\begin{subequations}
\begin{align}
\frac{d Q_L}{d P} & = \frac{\gbar(Y)}{E_P\bigl[\gbar(Y)\bigr]},\\
\frac{d Q_C}{d Q_L} & = \frac{\hbar(X)}{E_{Q_L}\bigl[\hbar(X)\bigr]}.
\end{align}
Then the following statements hold true:
\begin{itemize}
\item[(i)] $Q_L$ and $Q_C$ are well-defined since it holds that $E_P\bigl[\gbar(Y)\bigr] >0$ and\\
$E_{Q_L}\bigl[\hbar(X)\bigr] > 0$.
\item[(ii)] $Q = Q_C$.
\item[(iii)] For all $F \in \mathcal{F}$ it holds that
$Q_L[F \,|\,\sigma(Y)]  =  P[F \,|\,\sigma(Y)]$ and $Q_C[F \,|\,\sigma(X)] \ = \ Q[F \,|\,\sigma(X)]$.
In other words, $P$ and $Q_L$ are related through label shift while $Q_C$ and $Q$ are related through covariate 
shift.
\end{itemize}
\end{subequations}
\end{proposition}

See Appendix~\ref{se:proofs} for a proof of Proposition~\ref{pr:combined}.
Against the background of the description of FJS as a sequence of covariate and label shifts, 
the question arises how much FJS is constrained when compared to non-specific joint distribution 
shift between $Q$ and $P$.
The following result helps to answer this question by presenting the relations between 
the functions $\hbar$ and $\gbar$ which characterise FJS on the one hand and the feature density $h$
and the label density $g$ on the other hand.

\begin{theorem}\label{th:charac}
Assume the setting of Section~\ref{se:setting}.
\begin{enumerate}
\item[(i)] Let $P$ and $Q$ be related through FJS in the 
sense of Definition~\ref{de:FJS} such that there are an
$\mathcal{H}$-measurable function $\hbar: \Omega_{X} \to [0,\infty)$ and a $\mathcal{G}$-measurable 
function $\gbar: \Omega_{Y} \to [0,\infty)$ with the property 
that $\hbar(X)\,\gbar(Y)$ is a density of $Q$ with respect to $P$.
Denote by $h$ a density of $Q_X$ with respect to $P_X$ on $\mathcal{H}$ and by $g$
a density of $Q_Y$ with respect to $P_Y$ on $\mathcal{G}$.
Then it follows that $P$-a.s.
\begin{equation}\label{eq:conditions}
h(X) =  \hbar(X)\,E_P[\gbar(Y)\,|\,\sigma(X)]\qquad \text{and}\qquad g(Y) = \gbar(Y)\,E_P[\hbar(X)\,|\,\sigma(Y)].
\end{equation}
\item[(ii)] Suppose that $h$ is a density of some
probability measure $Q_X^\ast$ on $(\Omega_X, \mathcal{H})$ with respect to $P_X$ and that $g$ 
is a density of some
probability measure $Q_Y^\ast$ on $(\Omega_Y, \mathcal{G})$ with respect to $P_Y$.
In addition, let $\hbar: \Omega_{X} \to [0,\infty)$ be an $\mathcal{H}$-measurable function and 
$\gbar: \Omega_{Y} \to [0,\infty)$ be a $\mathcal{G}$-measurable 
function such that \eqref{eq:conditions} holds true. 
Then $\hbar(X)\,\gbar(Y)$ is a probability density on $(\Omega, \mathcal{F})$, and for $Q$ defined by
$\frac{d Q}{d P} = \hbar(X)\,\gbar(Y)$ it holds that $Q_X = Q_X^\ast$ and $Q_Y = Q_Y^\ast$.
\end{enumerate}
\end{theorem}

\emph{Proof of Theorem~\ref{th:charac}.}
\begin{itemize}
\item[(i)]
Assume that $\hbar(X)\,\gbar(Y)$ is a density of $Q$ with respect to $P$. Since $h(X)$ is a density of
$Q|\sigma(X)$ with respect to $P|\sigma(X)$ Lemma~\ref{le:density} implies that
\begin{equation*}
h(X) \ =\  E_P\bigl[\hbar(X)\,\gbar(Y)\,|\,\sigma(X)\bigr] \ = \ 
	\hbar(X)\,E_P\bigl[\gbar(Y)\,|\,\sigma(X)\bigr]. 
\end{equation*}
The second equation in \eqref{eq:conditions} follows similarly because $g(Y)$ is 
a density of $Q|\sigma(Y)$ with respect to $P|\sigma(Y)$.
\item[(ii)] Suppose that \eqref{eq:conditions} is true for densities $h$ and $g$ and functions $\hbar$ and $\gbar$ 
as described in Theorem~\ref{th:charac}~(ii). Then it holds that
\begin{equation*}
E_P[\hbar(X)\,\gbar(Y)] = E_P\bigl[\hbar(X)\,E_P[\gbar(Y)\,|\,\sigma(X)]\bigr]
	= E_P[h(X)] = 1.
\end{equation*}
Hence $\frac{d Q}{d P} = \hbar(X)\,\gbar(Y)$ defines a probability measure $Q$ on $(\Omega, \mathcal{F})$.
\eqref{eq:conditions} implies that $h(X)$ is the marginal density of $Q$ with respect to $P$ on $\sigma(X)$ and
$g(Y)$ is the marginal density of $Q$ with respect to $P$ on $\sigma(Y)$. From this observation it follows
that $Q_X = Q^\ast_X$ on $(\Omega_X, \mathcal{H})$ and $Q_Y = Q^\ast_Y$ on $(\Omega_Y, \mathcal{G})$. 
\hfill \qed
\end{itemize}

\begin{remark}\label{rm:generalisation}\emph{%
Under Assumption~\ref{as:discrete} `Categorical label space' from Appendix~\ref{se:further} below, 
Theorem~\ref{th:charac} can be reconciled 
with Theorem~2 of Tasche~\cite{tasche2022factorizable} as follows:
\begin{itemize}
\item Transformed into the notation of Theorem~\ref{th:charac}, (9a) of Tasche~\cite{tasche2022factorizable} reads 
$b = \gbar(Y)$, 
with $\frac{Q[Y=y]}{P[Y=y]} = g(y)$ and $\varrho(y) = \frac{\gbar(y)\,g(d)}{\gbar(d)\,g(y)}$
for $y = 1, \ldots, d$.
\item In the notation of Theorem~\ref{th:charac}, when taking into account that $g$ in  
Tasche~\cite{tasche2022factorizable} stands for $\hbar(X)$, $h$ for $h(X)$ and
$E_P\bigl[\gbar(Y)\,|\,X=x\bigr]  = \sum_{i=1}^d \gbar(i)\, P[Y=i\,|\,X=x]$, 
(9b) of Tasche~\cite{tasche2022factorizable} turns out to be a part of \eqref{eq:conditions}.
\item (9c) of Tasche~\cite{tasche2022factorizable} is mutatis mutandis nothing else but another part of
\eqref{eq:conditions}.\qed
\end{itemize} }
\end{remark}

\begin{remark}\label{rm:howConstrained}\emph{%
Suppose that $Q$ and $P$ are related through FJS and that Assumption~\ref{as:condDist}(i) is satisfied. 
Observe that then by \eqref{eq:conditions} it holds that $P_X$-a.s.
\begin{equation*}
\{h > 0\}\ = \ \{\hbar > 0\} \cap \bigl\{E_P[\gbar(Y)\,|\,X = \cdot] > 0 \bigr\}.
\end{equation*}
Therefore, it follows for $x\in\Omega_X$ with $h(x) > 0$  that $\frac{\hbar(x)}{h(x)} = 
\frac{1}{E_P[\gbar(Y)\,|\,X=x]}$. Accordingly, under FJS the conditional density $q_{Y|X}$ as 
defined in Theorem~\ref{th:condDens}~(i) can be
represented for all $x$ with $h(x) > 0$ and hence $Q_X$-a.s.\ as
\begin{subequations}
\begin{equation}\label{eq:condDensFJS}
q_{Y|X=x}(y) \ = \ \frac{\gbar(y)}{\int \gbar(z)\,P_{Y|X=x}(dz)}, \quad y \in \Omega_Y.
\end{equation}
Similarly, under Assumption~\ref{as:condDist}(ii) and an assumption of 
FJS the conditional density $q_{X|Y}$ as 
defined in Theorem~\ref{th:condDens}~(ii) can be
represented for all $y$ with $g(y) > 0$ and hence $Q_Y$-a.s.\ as
\begin{equation}\label{eq:condDensFJS.XY}
q_{X|Y=y}(x) \ = \ \frac{\hbar(x)}{\int \hbar(t)\,P_{X|Y=y}(dt)}, \quad x \in \Omega_X.
\end{equation}
Note that Eq.~(12) of Tasche~\cite{tasche2022factorizable} follows from \eqref{eq:condDensFJS} 
when it is adapted for categorical label spaces.
In the case where the label space $\Omega_Y$ is $\mathbb{R}$ or a subset of $\mathbb{R}$, \eqref{eq:condDensFJS} implies
the following formula for the expected value of $Y$ conditional on $X$ under the target distribution $Q$:
\begin{equation}\label{eq:prediction}
E_Q[Y\,|\,X=x] \ = \ \frac{\int y\,\gbar(y)\,P_{Y|X=x}(dy)}{\int \gbar(z)\,P_{Y|X=x}(dz)}, 
\quad Q_X\text{-a.s.\ for}\ x \in \Omega_X.
\end{equation}
\end{subequations}
Moreover, \eqref{eq:condDensFJS} suggests the following interpretation of the density factors $\hbar$ and $\gbar$ in
the definition of FJS:
\begin{itemize}
\item In general, $\gbar$ does not represent a marginal density of the label variable $Y$ as in
the case of label shift (i.e.\ when $\hbar = 1$) 
but rather is characterised by being a conditional density of $Q_{Y|X=x}$ with
respect to $P_{Y|X=x}$ for all $x\in\Omega_X$, after normalisation. Due to the symmetry of the setting in $X$ and $Y$, 
$\hbar$ is not a marginal density of the feature variable $X$ as in
the case of covariate shift but rather is characterised by being a conditional density of $Q_{X|Y=y}$ with
respect to $P_{X|Y=y}$ for all $y\in\Omega_Y$, again after normalisation.
\item Consequently, FJS is more homogeneous than non-specific distribution shift in so far as 
the target conditional densities of the labels given the features are all equal (up to normalisation). 
Based on a similar observation, Tasche~\cite{tasche2022factorizable} (Remark~1) suggested  talking about 
``scaled density ratios'' shift instead of FJS.\qed
\end{itemize}
}\end{remark}

\subsection{Solving Problem~\ref{prB} by assuming covariate shift}
\label{se:covariate}

A straightforward way to solving Problem~\ref{prB} as formulated in Section~\ref{se:shift} is
by assuming that the target distribution $Q$ and the source distribution $P$ are related through
covariate shift in the sense of Definition~\ref{de:covariatelabelShift}. 
Covariate shift is implied by ``missing at random'' sample selection bias 
(Moreno-Torres et al.~\cite{MorenoTorres2012521}, Section~6.1) 
such that there are applications where covariate shift is the only reasonable assumption.

Problem~\ref{prB} means that a target feature distribution $Q^\ast_X$ is pre-specified by its density $h$
with respect to the source feature distribution $P_X$. According to 
Proposition~\ref{pr:SpecCases}, then letting $\frac{d Q}{d P}  = h(X)$ defines a joint distribution $Q$ 
of features $X$ and labels $Y$ such that $Q$ and $P$ are related through covariate shift and 
it holds that $Q_X = Q^\ast_X$. Thus,  Problem~\ref{prB} is solved in principle.

However, Storkey~\cite{storkey2009training} (Section~5.1) commented that ``$\ldots$ 
there is some benefit to be obtained by doing some-thing different in the case of covariate shift. 
The argument here is that these papers indicate a computational benefit rather than a fundamental modelling benefit.''
The computational aspects of covariate shift have been extensively studied in the literature and are
not discussed in this paper. See for example
Section~3.1 of Kouw and Loog~\cite{Kouw&Loog2019} for a review of the 
literature on domain adaptation under covariate shift.

\subsection{Solving Problem~\ref{prB} by assuming label shift or constrained FJS}
\label{se:probB}

Recall that in Problem~\ref{prB} as 
phrased in Section~\ref{se:shift} only the feature marginal distribution $Q_X$ (or its 
density $h$ with respect to $P_X$) is known of the target distribution $Q$. While most of the time
the aim in Problem~\ref{prB} might be to come up with an estimate of the target posterior distribution $Q_{Y|X}$
of the labels conditional on the features, thanks to the law of total probability the marginal label
distribution $Q_Y$ then is a byproduct. In the literature, the quest for $Q_Y$ in the case
of categorical labels is called `quantification',
`counting', `class probability re-estimation', `class
prior estimation' or `class distribution estimation' (Esuli et al.~\cite{esuli2023learning}, p.~2).
In the case of general normed label spaces,  Azizzadenesheli~\cite{Azizzadenesheli2022} referred to the estimation
of $Q_Y$ as `Importance Weight Estimation'. In the case of non-categorical label spaces, the estimation of
$Q_Y$ also might be referred to as \emph{label distribution estimation}.

Theorem~\ref{th:charac}~(ii) suggests approaches to constructing a target distribution $Q$ related through
FJS to the source distribution $P$ that matches a predefined marginal feature distribution $Q^\ast_X$ and / or
a predefined marginal label distribution $Q^\ast_Y$, both specified through their densities $h$ and $g$ 
with respect to $P_X$ and $P_Y$ respectively. Theorem~\ref{th:charac}~(ii) does not
make any statement on the existence or uniqueness of the FJS-density factors $\gbar(Y)$ and $\hbar(X)$ 
that satisfy \eqref{eq:conditions}. Intuitively, one might nonetheless expect that the full range of
possibilities for $\gbar(Y)$ and $\hbar(X)$ be needed to find any solutions to \eqref{eq:conditions} when
both marginal distributions $Q^\ast_X$ and $Q^\ast_Y$ are predefined, i.e.\ in the situation of Problem~\ref{prA}
of Section~\ref{se:shift}. See Section~\ref{se:probA} below for a discussion of this case.

In contrast, one might expect that there is more, perhaps many more, than one pair of density factors 	
$\gbar(Y)$ and $\hbar(X)$ whose corresponding target distribution $Q$ is related to $P$ through FJS when
only -- say -- the marginal feature distribution $Q^\ast_X$ is predefined, i.e.\ in the situation of Problem~\ref{prB}
of Section~\ref{se:shift}. From previous work by the author, it is known that in the case of categorical labels
an assumption of label shift or constrained FJS then restricts the number of solutions to \eqref{eq:conditions}
in a reasonable way (Sections~4.1 and 4.2 of Tasche~\cite{tasche2022factorizable}; 
Theorem~3 of Tasche~\cite{tasche2014exact}).
In the following, the assumption of FJS is constrained by the requirement that the density factor $\gbar(Y)$
is identical with the unknown label density $g(Y)$, i.e.\ it holds that $\gbar(Y) = g(Y)$ in \eqref{eq:conditions}. 
This includes label shift ($\hbar(X) = 1$) as well as more general FJS which is subject to the constraint
\begin{equation}\label{eq:constraint}
1\ = \ E_P[\hbar(X)\,|\,\sigma(Y)].
\end{equation}

Suppose that in Theorem~\ref{th:charac}~(ii) the feature density $h$ is positive. Then \eqref{eq:conditions} 
together with the constraint \eqref{eq:constraint}
suggests the following alternating iteration for constructing the label density $g$ 
for an FJS-based solution to Problem~\ref{prB} as phrased in Section~\ref{se:shift}:
\begin{subequations}
\begin{enumerate}
\item[1)] Choose an initial label density $g_0(Y) > 0$, for instance $g_0(Y) = 1$
(i.e.\ assume $Q_Y = P_Y$ for the initialisation\footnote{%
$g_0 = 1$ is the only possible choice in the absence of any specific information about $P_Y$. 
If specific knowledge about $P_Y$ is available different initialisations might make sense to
accelerate convergence or avoid saddle points in sample-based maximum likelihood estimation.}).
\item[2)] For a given label density $g_n(Y)$, compute 
\begin{equation}\label{eq:hnB}
h_n(X) \ =\ E_P[g_n(Y)\,|\,\sigma(X)].
\end{equation}
\item[3)] For given $h_n(X)$, compute 
\begin{equation}\label{eq:gnB}
g_{n+1}(Y) \ =\  g_n(Y)\,E_P\left[\frac{h(X)}{h_n(X)}\,\Big|\,\sigma(Y)\right].
\end{equation}
\item[4)] Repeat 2) and 3) for $n = 0, 1, \ldots$ until an appropriate stop criterion is fulfilled.
\end{enumerate}
\end{subequations}

The following theorem lists some properties that make iteration \eqref{eq:hnB} and \eqref{eq:gnB} work.
Theorem~\ref{th:alternateB} generalises the appendix of Saerens et al.~\cite{saerens2002adjusting} 
``Derivation of the EM Algorithm''
from categorical label spaces to general label spaces. Thus, for non-categorical label spaces,
Theorem~\ref{th:alternateB} establishes the EM algorithm as a potential alternative to the 
`inverse operator' approach 
(comparable to `adjusted classify \& count' in the case of categorical labels) by 
 Azizzadenesheli~\cite{Azizzadenesheli2022} and the
segmentation and discretisation approaches by Bella et al.~\cite{bella2014aggregative}.
Corollary~\ref{co:KL} below suggests a stop criterion for the iteration \eqref{eq:hnB} and \eqref{eq:gnB}.

\begin{theorem}\label{th:alternateB}
In the setting of Theorem~\ref{th:charac}~(ii), assume additionally $h >0$. Choose some
$g_0(Y) > 0$ with $E_P[g_0(Y)] = 1$ and define
$h_n$ and $g_{n+1}$  for $n \ge 0$ recursively by \eqref{eq:hnB} and \eqref{eq:gnB}. 
Then the following conclusions hold true:
\begin{itemize}
\item[(i)] $0 < h_n(X)< \infty$ and $0 < g_{n+1}(Y) < \infty$ $P$-a.s.\ for all $n \ge 0$.
Moreover,  $g_n(Y)$ and $h_n(X)$ are probability
densities under the source distribution $P$ for all $n$.
\item[(ii)] Define $f_n = f_n(X,Y) = \frac{h(X)\,g_n(Y)}{h_n(X)}$ for $n \ge 0$. Then
$f_n$ is a probability density under $P$ for all $n$.
\item[(iii)] Define the probability measures $Q^{(n)}$ and $R^{(n)}$ on $(\Omega, \mathcal{F})$, 
$Q^{(n)}$ related to $P$ through label shift and $R^{(n)}$ related to $P$ through FJS, by
\begin{equation*}
\frac{d Q^{(n)}}{d P} = g_n(Y) \quad\text{and}\quad \frac{d R^{(n)}}{d P} = f_n(X,Y).
\end{equation*}
Then it holds for all $n$ that $\frac{d Q_Y^{(n)}}{d P_Y} = g_n$, $\frac{d Q_X^{(n)}}{d P_X} = h_n$, 
$\frac{d R_X^{(n)}}{d P_X} = h$ and $\frac{d R_Y^{(n)}}{d P_Y} = g_{n+1}$.
\item[(iv)] Assume that\footnote{%
The integrability assumptions hold true if there are constants $0 < c < C < \infty$ such
that $c \le h(X) \le C$ and $c \le g_0(Y) \le C$.}
$E_P[h(X)\,\lvert\log(h_n(X))\rvert] < \infty$ and $E_P[f_n\,\lvert\log(g_k(Y))\rvert] < \infty$
for all $n\ge 0$ and $k \in\{n, n+1\}$ as well as $E_P[h(X)\,\lvert\log(h(X))\rvert] < \infty$. 
Then it follows for all $n$ that 
\begin{equation}\label{eq:ineqEnt}
E_P[h(X)\,\log(h_n(X))] \ \le \ E_P[h(X)\,\log(h_{n+1}(X))].
\end{equation}
\end{itemize}
\end{theorem}

See Appendix~\ref{se:proofs} for a proof of Theorem~\ref{th:alternateB}. In the following analyses of the algorithm specified by \eqref{eq:hnB} and \eqref{eq:gnB}, 
we make use of the notion of Kullback-Leibler divergence 
(e.g.~Section~3.2 of Reid and Williamson~\cite{reid2011information}). 

\begin{definition}\label{de:KL}
Let $(\Omegabar, \mathcal{M}, \mu)$ be a probability space. For $\mathcal{M}$-Borel measurable
functions $\lambda_0 \ge 0$ and $\lambda_1 \ge 0$ with $\int \lambda_i\,d\mu =1$, 
$\mu[\lambda_i = 0] = 0$ for $i=0,1$ and 
$\int \lambda_0 \left\vert \log\left(\frac{\lambda_0}{\lambda_1} \right)\right\vert d\mu < \infty$, 
the \emph{Kullback-Leibler
divergence} $\mathrm{KL}_\mu(\lambda_0\parallel\lambda_1)$ of $\lambda_1$ with respect to $\lambda_0$ is defined as
$\mathrm{KL}_\mu(\lambda_0\parallel\lambda_1) = 
\int \lambda_0\,\log\left(\frac{\lambda_0}{\lambda_1} \right) d\mu$.
\end{definition}

Jensen's inequality implies $\mathrm{KL}_\mu(\lambda_0\parallel\lambda_1) \ge 0$ and
$\mathrm{KL}_\mu(\lambda_0\parallel\lambda_1) = 0$ if and only if $\lambda_0 = \lambda_1$ $\mu$-a.s.\
Accordingly, $\lambda_0$ and $\lambda_1$ are the more similar, the closer to naught 
$\mathrm{KL}_\mu(\lambda_0\parallel\lambda_1)$ is.
Minimising $\mathrm{KL}_\mu(\lambda_0\parallel\lambda_1)$
by changing $\lambda_1$ may be interpreted as a population version of maximum likelihood estimation.
In this sense, Theorem~\ref{th:alternateB}~(iv) can be understood as a statement about
approximation on the basis of $\mathrm{KL}$.

\begin{corollary}\label{co:KL}
In the setting of Theorem~\ref{th:alternateB}~(iv), it follows that
\begin{equation}\label{eq:KLineq}
\mathrm{KL}_{P_X}(h\parallel h_{n+1})\ \le \ \mathrm{KL}_{P_X}(h\parallel h_n),\quad \text{for all}\ 
n = 0, 1, \ldots.
\end{equation}
\end{corollary}
Corollary~\ref{co:KL} suggests stopping iteration \eqref{eq:hnB} and \eqref{eq:gnB} when
\begin{equation}\label{eq:stop}
\mathrm{KL}_{P_X}(h\parallel h_n) - \mathrm{KL}_{P_X}(h\parallel h_{n+1}) \ <\ \varepsilon
\end{equation}
for some pre-specified $\varepsilon > 0$ has been attained.

How to reconcile the iteration \eqref{eq:hnB} and \eqref{eq:gnB} and Theorem~\ref{th:alternateB} with
previous work on the EM algorithm for class distribution estimation 
(in particular, Saerens et al.~\cite{saerens2002adjusting})?
\eqref{eq:hnB} and \eqref{eq:gnB} and Theorem~\ref{th:alternateB} are valid for the general setting
of Section~\ref{se:setting}. For the reconciliation, it is useful to re-phrase \eqref{eq:hnB} and \eqref{eq:gnB}
successively under the ever more special assumptions which are presented below in Appendices~\ref{se:source} and
\ref{se:further}: 

\begin{subequations}
\emph{Regular conditional distributions (Assumption~\ref{as:condDist}).} \eqref{eq:hnB} becomes
\begin{align}
h_n(x) & = \int g_n(y)\,P_{Y|X=x}(d y),\label{eq:hnReg}\\
\intertext{and \eqref{eq:gnB} reads as}
g_{n+1}(y) & = g_n(y) \int \frac{h(x)}{h_n(x)}\,P_{X|Y=y}(d x).\label{eq:gnReg}
\end{align}
\end{subequations}

\begin{subequations}
\emph{Absolute continuity of source distribution (Assumption~\ref{as:aux}).} \eqref{eq:hnB} becomes
\begin{align*}
h_n(x) & = \int g_n(y)\,\varphi(x,y)\,P_Y(d y),\\
\intertext{and \eqref{eq:gnB} reads as}
\begin{split}
g_{n+1}(y) & = g_n(y) \int \frac{h(x)\,\varphi(x,y)}{h_n(x)}\,P_X(d x)\\
& = g_n(y) \int \frac{\varphi(x,y)}{h_n(x)}\,Q_X(d x).
\end{split}
\end{align*}
\end{subequations}

\begin{subequations}
\emph{Categorical label space (Assumption~\ref{as:discrete}).} With $\varphi(x,y) = \frac{P[Y=y\,|\,X=x]}{P[Y=y]}$
according to \eqref{eq:discreteDensity}, \eqref{eq:hnB} becomes
\begin{align}
h_n(x) & = \sum_{y=1}^d g_n(y)\,P[Y=y\,|\,X=x],\notag\\
\intertext{and \eqref{eq:gnB} reads as}
g_{n+1}(y) & = g_n(y) \int \frac{h(x)\,P[Y=y\,|\,X=x]}{h_n(x)\,P[Y=y]}\,P_X(dx)\notag\\
& = g_n(y) \int \frac{P[Y=y\,|\,X=x]}{h_n(x)\,P[Y=y]}\,Q_X(dx).\notag
\intertext{By Theorem~\ref{th:alternateB}~(iii), it holds that $g_n(y) = \frac{Q^{(n)}[Y=y]}{P[Y=y]}$. This
implies}
\frac{Q^{(n+1)}[Y=y]}{P[Y=y]} & = \frac{Q^{(n)}[Y=y]}{P[Y=y]} 
	\int \frac{P[Y=y\,|\,X=x]}{P[Y=y] \sum\limits_{z=1}^d \frac{Q^{(n)}[Y=z]}{P[Y=z]} \,P[Y=z\,|\,X=x]}\,Q_X(dx)\notag\\
\intertext{and equivalently}
Q^{(n+1)}[Y=y] & = \frac{Q^{(n)}[Y=y]}{P[Y=y]} 
	\int \frac{P[Y=y\,|\,X=x]}{\sum\limits_{z=1}^d \frac{Q^{(n)}[Y=z]}{P[Y=z]} \,P[Y=z\,|\,X=x]}\,Q_X(dx).
	\label{eq:saerens}
\end{align}
Substituting a sample-based empirical measure $\widetilde{Q}_X = \frac{1}{n}\sum_{i=1}^n \delta_{x_i}$ for $Q_X$ in
\eqref{eq:saerens} yields the EM algorithm (2.9) of Saerens et al.~\cite{saerens2002adjusting}, i.e.\
for $y = 1, \ldots, d$
\begin{equation}\label{eq:saerensEmp}
\widetilde{Q}^{(n+1)}[Y=y] \ =\ \frac{\widetilde{Q}^{(n)}[Y=y]}{n\,P[Y=y]} 
	\sum_{i=1}^n \frac{P[Y=y\,|\,X=x_i]}{\sum_{z=1}^d 
	\frac{\widetilde{Q}^{(n)}[Y=z]}{P[Y=z]} \,P[Y=z\,|\,X=x_i]}.
\end{equation}
\end{subequations}

\begin{remark}[Interpretation of Theorem~\ref{th:alternateB}]\emph{%
Compared to (2.9) of Saerens et al.~\cite{saerens2002adjusting}, it is less obvious how to interpret 
\eqref{eq:hnB} and \eqref{eq:gnB}
in terms of expectation-maximisation. Instead, thanks to Theorem~\ref{th:alternateB}~(iii), 
the algorithm specified by \eqref{eq:hnB} and \eqref{eq:gnB} can be interpreted as rendering two
solutions for Problem~\ref{prB} from Section~\ref{se:shift}:
\begin{enumerate}
\item[1)] An approximate solution $Q^{(n)}$ based on a label shift assumption. As a consequence of Corollary~\ref{co:KL},
each further step of the algorithm provides
an approximate target distribution whose marginal feature density $h_{n+1}$ is more similar to the
predefined feature density $h$ than the feature density $h_n$ of the previous approximate target distribution.
\item[2)] An exact fit solution $R^{(n)}$ related to the source distribution $P$ through FJS, 
with the same label distribution as the label shift approximation.\qed
\end{enumerate}
} 
\end{remark}

Although the decreasing sequence $\mathrm{KL}_{P_X}(h\parallel h_n)$ converges to a limit $\mathrm{KL}^\ast$
for $n \to \infty$, 
Theorem~\ref{th:alternateB} and Corollary~\ref{co:KL} do not exclude the possibility that $\mathrm{KL}^\ast > 0$ 
which would suggest that
the predefined target feature density $h$ cannot be realised by any marginal target feature density 
induced by a target joint distribution related to the source distribution $P$ through `pure'
label shift.
Neither does Theorem~\ref{th:alternateB} imply that there is a 
density $h^\ast$ with $\mathrm{KL}_{P_X}(h\parallel h^\ast) = \mathrm{KL}^\ast$ or that the $Q^{(n)}$
converge to any limit distribution. For the case of a categorical label space, Peters and 
Coberly~\cite{peters1976numerical} and Wu~\cite{Wu1983EM}
provided sufficient conditions for the sequence $g_n$, $n = 0, 1, \ldots$, of target label densities 
to converge (which would entail convergence of the related target joint distributions $Q^{(n)}$).

For the case of a general label space, the following Proposition~\ref{pr:convergence} presents consequences of 
the potential convergence of the target label densities $g_n$ which are constructed with the variant of the 
EM algorithm as described by \eqref{eq:hnB} and \eqref{eq:gnB}. \\
Recall the definitions of \emph{convergence in probability} and \emph{convergence in mean}:
\begin{itemize}
\item A sequence of real-valued random variables $Z_n$, $n=0, 1, \ldots$, converges in probability under
a probability measure $P$ to a real-valued random variable $Z$ if it holds that
$\lim_{n\to\infty} P[\lvert Z - Z_n\rvert \ge \varepsilon] = 0$ for all $\varepsilon > 0$ 
(see for instance Klenke~\cite{klenke2013probability}, Definition~6.2).
\item A sequence of integrable real-valued random variables $Z_n$, $n=0, 1, \ldots$, converges in mean under
a probability measure $P$ to an integrable real-valued random variable $Z$ if it holds that
$\lim_{n\to\infty} E_P[\lvert Z - Z_n\rvert] = 0$ (Klenke~\cite{klenke2013probability}, Definition~6.8).
\end{itemize}
Convergence in mean implies convergence in probability.

\begin{proposition}\label{pr:convergence}
In addition to the assumptions and definitions of Theorem~\ref{th:alternateB}, suppose that the densities $g_n$ 
with respect to $P_Y$ defined in Theorem~\ref{th:alternateB} converge in probability to a random variable $g> 0$. 
Furthermore, assume\footnote{%
This assumption on the exchangeability of integration and $\lim$ in particular holds true if there is 
some $M > 0$ such that $g_n \le M$ for all $n = 0, 1, \ldots$ and $y \in \Omega_Y$.}  that $\int g(y)\,P_Y(dy) =1$ 
such that $g$ also is a density with respect to $P_Y$.\\
Define a joint distribution $\Qhat$ of $X$ and $Y$ related to $P$ through label shift by
$\frac{d \Qhat}{d P} = g(Y)$  (such that $\frac{d \Qhat_Y}{d P_Y} = g$ follows).
Let $\hhat(x) = \frac{d \Qhat_X}{d P_X}(x) = E_P[g(Y)\,|\,X =x]$ for $x \in \Omega_X$ and define
the feature densities $h_n = E_P[g_n(Y)\,|\,X =\cdot]$ of the joint distributions $Q^{(n)}$ as in 
Theorem~\ref{th:alternateB} by \eqref{eq:hnB}.\\
Then the following statements hold true:
\begin{enumerate}
\item[(i)] $g_n$ converges also in mean to $g$ under $P_Y$,
\item[(ii)] $\hhat > 0$, and $h_n$ converges in mean to $\hhat$ under $P_X$,
\item[(iii)] $\displaystyle\mathrm{KL}_{P_X}(h \parallel \hhat) \ \le\ \lim\limits_{n\to\infty} 
\mathrm{KL}_{P_X}(h \parallel h_n)$,
\item[(iv)] 
$\displaystyle{} E_P\left[\frac{h(X)}{\hhat(X)}\,\Big|\,\sigma(Y)\right] \ \le \ 1$ $P$-a.s.
\end{enumerate}
\end{proposition}
See Appendix~\ref{se:proofs} for a proof of Proposition~\ref{pr:convergence}.

\begin{remark}\label{rm:achieved} \emph{Suppose that $E_P\left[\frac{h(X)}{\hhat(X)}\,\Big|\,\sigma(Y)\right] = 1$ 
holds in Proposition~\ref{pr:convergence}~(iv). Then $\hbar(X) = \frac{h(X)}{\hhat(X)}$ satisfies
\eqref{eq:constraint} and by Theorem~\ref{th:charac}~(ii) $\frac{d Q}{d P} = \frac{g(Y)\,h(X)}{\hhat(X)}$ defines
a target distribution $Q$
related to $P$ through FJS in the sense of Definition~\ref{de:FJS}. Moreover, it holds
that $Q_X = Q_X^\ast$ for the predefined feature distribution $Q_X^\ast$ from Theorem~\ref{th:charac}~(ii).
Hence Problem~\ref{prB} of Section~\ref{se:shift} is
solved by constructing a target distribution $Q$ with the predefined
marginal feature distribution under an assumption of FJS composed by the two following elementary shifts:
\begin{enumerate}
\item A label shift based on the target label density $g$ resulting from the iteration \eqref{eq:hnB} and
\eqref{eq:gnB}.
\item A covariate shift based on the ratio $h / \,\hhat$ of the predefined or observed target feature 
density $h$ and the target feature density $\hhat = E_P[g(Y)\,|\,X=\cdot]$ implied by the previous label shift,
without any further change of the target posterior label distribution.
\end{enumerate}
Alternatively, $\frac{d Q_Y}{d P_Y} = g$ may be considered a solution to the label distribution estimation
problem for the given target feature density $h$ under the assumption of label shift between source and target.}  
\hfill$\qed$
\end{remark}

Proposition~\ref{pr:convergence} and Remark~\ref{rm:achieved} confirm for general label spaces 
that if the sequence of densities $g_n$
constructed with the EM algorithm converges to a positive limit density $g$, 
then this density $g$ is a promising candidate to solve \eqref{eq:conditions}
under the constraint \eqref{eq:constraint} and as a consequence through the FJS-relation between $P$ and $Q$ 
also Problem~\ref{prB} as formulated in Section~\ref{se:shift}.  
Thereby, Remark~\ref{rm:achieved} resumes for general label spaces the `exact fit' 
property proven by Tasche~\cite{tasche2014exact} (Remark~1) for the case of 
categorical labels. `Exact fit' as discussed by Tasche~\cite{tasche2014exact} means that there is a target joint distribution 
such that the ratios of pairs of feature densities
conditional on the labels do not depend on feature values (cf.~Remark~\ref{rm:howConstrained} above regarding this
property for FJS in the case of non-categorical labels) and that the target marginal feature distribution 
exactly matches a pre-specified or observed feature distribution. 

Note that exact fit in the sense of matching
a target feature distribution with a constrained FJS (by requiring \eqref{eq:constraint}) is
not always possible as pointed out by Titterington et al.~\cite{titterington1985statistical} 
(Example~4.3.1) in the context of
maximum likelihood estimation of class prior probabilities in the case of binary labels. This observation 
does not contradict Proposition~\ref{pr:convergence} or Remark~\ref{rm:achieved} because in cases of 
imperfect fit one or more of the conditions 
listed in the proposition and the remark would be violated.

Observe that the pre-specified target feature distribution $Q^\ast_X$ in the shape of its $P_X$-density $h$
is a factor impacting the recursive definition \eqref{eq:gnB} of the label densities $g_n$ which are
constructed when running the EM algorithm. This observation raises the question if the limit feature 
density $\hhat$ in Proposition~\ref{pr:convergence}  
perhaps in many cases is identical with the pre-specified feature density $h$, at
least if the label densities $g_n$ converge to a positive limit density such that all conditions 
of Proposition~\ref{pr:convergence} are satisfied? In other words, should one expect 
$\mathrm{KL}_{P_X}(h \parallel \hhat)=0$ (or equivalently $h = \hhat$) in Proposition~\ref{pr:convergence}?
If so, exact fit would appear to be achievable with label shift
quite generally. However, the example of normal mixture target feature 
distributions on the real line with $N+1$ components shows that
exact fit with label shift in general is impossible because no $(N+1)$-modal target feature density on the real line 
can be exactly matched with an $N$-components normal mixture density (in the case where the $N$ source class-conditional
feature distributions are normal). Hence, often the 
term $\mathrm{KL}_{P_X}(h \parallel \hhat)$ in Proposition~\ref{pr:convergence} 
will be positive such that the fit of the
pre-specified feature distribution with the limiting label shift target distribution is imperfect.

Does imperfect fit of the pre-specified feature distribution by an approximate solution of Problem~\ref{prB} 
matter in practice?
\begin{itemize}
\item The answer is no if it is the target \emph{posterior label distribution} given the features
 that is important. This
assessment follows from \eqref{eq:condDensFJS} that shows that the target posterior label distribution 
given the features is the same for label shift and for more general FJS 
as long as \eqref{eq:constraint} is fulfilled (implying $\gbar = g$) and the target label distribution is the same.
\item The answer is yes if it is the target \emph{posterior feature distribution} given the label that is important.
Under label shift (i.e.\ when the fit is possibly imperfect) the target and source posterior 
feature distributions given the label
are equal, i.e.\ it holds that $q_{X|Y}=1$ for the conditional densities 
defined in Theorem~\ref{th:condDens}~(ii) below. This need not be the case when there is exact fit as in 
Remark~\ref{rm:achieved} thanks to FJS subject to \eqref{eq:constraint}. This observation 
is a consequence of Remark~\ref{rm:howConstrained} because under FJS according to \eqref{eq:condDensFJS.XY}
the target  posterior feature densities given the label may observably differ as long as 
their ratios are functions of the label only. 
See Figure~1 of Tasche~\cite{tasche2017fisher} for an example of this phenomenon.
\end{itemize}

\subsection{Solving Problem~\ref{prA} by assuming FJS}
\label{se:probA}

Suppose that in Theorem~\ref{th:charac}~(ii) the densities $g$ and $h$ are positive. Then, similarly to
the approach in Section~\ref{se:probB} above, \eqref{eq:conditions} 
suggests the following alternating iteration for constructing the density factors $\gbar$ and $\hbar$ 
for an FJS-based solution to Problem~\ref{prA} as phrased in Section~\ref{se:shift}:
\begin{subequations}
\begin{enumerate}
\item[1)] Choose\footnote{%
Note that for $\widetilde{g}_0(Y) = c\,\gbar_0(Y)$ with $c>0$ constant and
$\widetilde{h}_n(X)$ and $\widetilde{g}_{n+1}(y)$ defined by \eqref{eq:hn} and \eqref{eq:gn} with
$\gbar_0$ replaced with $\widetilde{g}_0$, it follows that $\widetilde{g}_n(Y) =c\,\gbar_n(Y)$ and
$\widetilde{h}_n(X) = \hbar_n(X) / c$ for all $n \ge 0$. Thus the ambiguity exposed in \eqref{eq:hbarc}
re-emerges in the proposed iteration.
} $\gbar_0(Y) > 0$, for instance $\gbar_0(Y) = g(Y)$.
\item[2)] For given $\gbar_n(Y)$, compute 
\begin{equation}\label{eq:hn}
\hbar_n(X) \ =\ \frac{h(X)}{E_P[\gbar_n(Y)\,|\,\sigma(X)]}.
\end{equation}
\item[3)] For given $\hbar_n(X)$, compute 
\begin{equation}\label{eq:gn}
\gbar_{n+1}(Y) \ =\  \frac{g(Y)}{E_P[\hbar_n(X)\,|\,\sigma(Y)]}.
\end{equation}
\item[4)] Repeat 2) and 3) for $n = 0, 1, \ldots$ until an appropriate stop criterion is fulfilled.
\end{enumerate}
\end{subequations}
Convergence of the sequences $\gbar_n(Y)$ and $\hbar_n(X)$ is not obvious, and no proof of convergence is provided 
here. However, if there is convergence then it is pausible
that \eqref{eq:hn} and \eqref{eq:gn} imply \eqref{eq:conditions}. Moreover, the following
proposition shows that the iteration \eqref{eq:hn} and \eqref{eq:gn} generates two sequences
of joint distributions which are related to the source distribution $P$ through FJS and 
match the pre-defined feature distribution (specified by the density $h$) or the
predefined label distribution (specified by the density $g$).

\begin{proposition}\label{pr:alternate}
In the setting of Theorem~\ref{th:charac}~(ii), assume additionally $g >0$ and $h >0$. Choose some
$0 < \gbar_0(Y) < \infty$ and define
$\hbar_n$ and $\gbar_{n+1}$  for $n \ge 0$ recursively by \eqref{eq:hn} and \eqref{eq:gn}. 
Then it follows that $0 < \hbar_n(X)< \infty$ and $0 < \gbar_{n+1}(Y) < \infty$ $P$-a.s.\ for all $n \ge 0$.
Moreover,  $\gbar_n(Y)\,\hbar_n(X)$ and $\gbar_{n+1}(Y)\,\hbar_n(X)$ are probability
densities under the source distribution $P$ for all $n$. \\
Define the probability measures $Q^{(n)}$ and $R^{(n)}$ on $(\Omega, \mathcal{F})$, 
both related to $P$ through FJS, by
\begin{equation*}
\frac{d Q^{(n)}}{d P} = \gbar_n(Y)\,\hbar_n(X)\quad\text{and}\quad \frac{d R^{(n)}}{d P} = \gbar_{n+1}(Y)\,\hbar_n(X).
\end{equation*}
Then it holds for all $n$ that $\frac{d Q_X^{(n)}}{d P_X} = h$ and $\frac{d R_Y^{(n)}}{d P_Y} = g$.
\end{proposition}

How to choose a stop criterion for the iteration specified by \eqref{eq:hn} and \eqref{eq:gn}?\\
As one hopes that $\hbar_n$ and $\gbar_n$ converge, it would be natural to stop the iteration 
when the difference of $\hbar_{n+1}$ and $\hbar_n$ and the difference of $\gbar_{n+1}$ and $\gbar_n$
both fall below some predefined threshold. However, $\hbar_n$ and $\gbar_n$ in general are not densities
such that selecting an appropriate metric for measuring the differences is not straightforward.
 
Observe that for the distributions $Q^{(n)}$ and $R^{(n)}$ defined in Proposition~\ref{pr:alternate} 
it holds that
\begin{equation}\label{eq:marginals}
\begin{split}
\frac{d Q_Y^{(n)}}{d P_Y}(y) & = \gbar_n(y)\,E_P[\hbar_n(X)\,|\,Y=y] = g_n(y), \quad\text{and}\\
\frac{d R_X^{(n)}}{d P_X}(x) & = \hbar_n(x)\,E_P[\gbar_{n+1}(Y)\,|\,X=x] = h_n(x).
\end{split}
\end{equation}
The aim with iteration \eqref{eq:hn} and \eqref{eq:gn} is to have the marginal densities $g_n$ and $h_n$
according to \eqref{eq:marginals}
converge to the predefined label density $g$ and feature density $h$ respectively. Hence it makes sense
to base the stop criterion on divergences between $h_n$ and $h$ and $g_n$ and $g$ respectively, 
analogously to \eqref{eq:stop}. 
Since in contrast to Section~\ref{se:probB} convergence of iteration \eqref{eq:hn} and \eqref{eq:gn}
cannot be guaranteed, the stop criterion must also include a maximum number of iteration steps.

In \eqref{eq:hn} and \eqref{eq:gn} the density factor $\hbar(X)$ actually is only an auxiliary 
variable that can be eliminated in order to arrive at another version of the iteration:
\begin{enumerate}
\item[1)] Choose $\psi_0(Y) > 0$, for instance $\psi_0(Y) = 1$.
\item[2)] For given $\psi_n(Y)$, compute 
\begin{equation}\label{eq:psin}
\psi_{n+1}(Y) \ =\  \left(E_P\left[\frac{h(X)}
	{E_P\bigl[\psi_n(Y)\,g(Y)\,|\,\sigma(X)\bigr]}\,\bigm|\,\sigma(Y)\right]\right)^{-1}.
\end{equation}
\item[3)] Repeat 2) for $n = 0, 1, \ldots$ until an appropriate stop criterion is fulfilled.
\end{enumerate}
With $\gbar_n(Y) = \psi_n(Y)\,g(Y)$ and \eqref{eq:hn} for $\hbar(X)$, the density factors appearing in the original 
iteration \eqref{eq:hn} and \eqref{eq:gn} can be reconstructed from $\psi_n(Y)$.

Under Assumption~\ref{as:aux} `Absolute continuity of source distribution', 
\eqref{eq:psin} can be simplified thanks to Proposition~\ref{pr:condDist} such
that for $y \in \Omega_Y$ the iteration reads:
\begin{equation}\label{eq:simplified}
\begin{split}
\psi_n(y) & =  \left(\int \frac{h(x)\,\varphi(x,y)}{\int \psi_n(z)\,g(z)\,\varphi(x,z)\,P_Y(dz)} 
				\,P_X(d x)\right)^{-1} \\
	& = \left( \int \frac{\varphi(x,y)}{\int \psi_n(z)\,\varphi(x,z)\,Q_Y(dz)}\, Q_X(d x)\right)^{-1}.
\end{split}
\end{equation}
On the one hand \eqref{eq:simplified} is useful because the integrals with respect to conditional distributions have been
replaced with integrals with respect to less complex marginal distributions. On the other hand, 
\eqref{eq:simplified} (thanks to the second row) allows to make use of observations of 
the target feature variable  and the target label variable if the integrals are approximated by sample means. 

Note that the existence of $\lim_{n\to\infty} \psi_n(y) = \psi^\ast(y)$ for all $y$ 
of the sequence defined by \eqref{eq:psin} does not necessarily always imply that $\psi^\ast$ brings about
a solution to \eqref{eq:conditions}
because without further conditions the possibility cannot be excluded that
\begin{equation}\label{eq:neqIntegral}
\lim\limits_{n\to\infty}  \int \frac{\varphi(x,y)}{\int \psi_n(z)\,\varphi(x,z)\,Q_Y(dz)}\, Q_X(d x)
\ \neq \ \int \frac{\varphi(x,y)}{\int \lim\limits_{n\to\infty} \psi_n(z)\,\varphi(x,z)\,Q_Y(dz)}\, Q_X(d x).
\end{equation}
\eqref{eq:neqIntegral} can even be an issue in the case of a categorical label space. For the case of binary labels,
Tasche~\cite{tasche2022factorizable} (Proposition~2) proved existence and uniqueness of a solution
to \eqref{eq:conditions} in terms of $\varrho(y) = \frac{\psi^\ast(y)}{\psi^\ast(y_0)} = 
\frac{\gbar(y)\,g(y_0)}{\gbar(y_0)\,g(y)}$ with some fixed $y_0 \in \Omega_Y$. 
 
\subsection{The EM algorithm in the case of finite categorical feature spaces}
\label{se:EMfinite}

In Section~\ref{se:probB}, a version of the  EM algorithm for class distribution estimation 
(Saerens et al.~\cite{saerens2002adjusting}, Eq.~(2.9)) has been formulated under general assumptions on the label and features spaces.
Then it has been shown that in theory this general EM algorithm  provides a sequence of label 
shifts to the source distribution
such that the Kullback-Leibler divergence between the predefined target feature distribution and
the implied target feature distributions decreases with each step of the algorithm. In the case where
after convergence of the algorithm the final implied target feature distribution does not yet fully match the
predefined target feature distribution, exact fit might be achieved by an additional covariate shift as
explained in Remark~\ref{rm:achieved}.

Consider the following equivalent representation of the EM algorithm as described in \eqref{eq:hnReg} and 
\eqref{eq:gnReg} for $n \ge 0$ and some initial $P_Y$-density $g_0$ under Assumption~\ref{as:condDist} from
Appendix~\ref{se:facilitating}:
\begin{subequations}
\begin{align}
h_n(x) & \ = \ \int g_n(z)\,P_{Y|X=x}(d z),\quad x \in \Omega_X,\label{eq:hnCat}\\
g_{n+1}(y) & \ = \ g_n(y) \int \frac{h(x)}{h_n(x)}\,P_{X|Y=y}(dx), \quad y \in \Omega_Y, \notag\\
& \ = \ g_0(y) \prod_{i=0}^n \int \frac{h(x)}{h_i(x)}\,P_{X|Y=y}(dx), \quad y \in \Omega_Y.
	\label{eq:EMconsolidated}
\end{align}
\end{subequations}
\eqref{eq:hnCat} and \eqref{eq:EMconsolidated} indicate the heavy computational burden due to 
iterated and nested integration
the EM algorithm could entail in the face of general label and feature spaces. Unsurprisingly, research on
the EM algorithm for class distribution estimation has focussed on the case of categorical label spaces
(Saerens et al.~\cite{saerens2002adjusting}, Esuli et al.~\cite{Esuli2021critical}). 

In  the case of categorical labels, by \eqref{eq:saerens}, the computational effort is significantly 
reduced by the disappearance
of the nested integration which is encountered in \eqref{eq:EMconsolidated}. Substituting the empirical distribution
of the observed target distribution for $Q_X$ as in \eqref{eq:saerensEmp}  further reduces the computational burden
such that the estimation of well-calibrated source label posterior
probabilities $P[Y=k\,|\,X=x]$, $y = k, \ldots, d$, $x\in \Omega_X$, becomes 
the main challenge with the algorithm (Alexandari et al.~\cite{alexandari2020maximum}, 
Esuli et al.~\cite{Esuli2021critical}).

Another way to avoid the nested integration in \eqref{eq:EMconsolidated} is to make an assumption for the feature space
analogous to Assumption~\ref{as:discrete} for the labels.
\begin{assumption}[Categorical feature space]\label{as:features}
In the setting of Section~\ref{se:setting}, the feature space $\Omega_X$ is finite
with $\Omega_X = \{1, \ldots, d\}$, $d \ge 2$. The $\sigma$-algebra
$\mathcal{H}$ is the power set of $\Omega_X$, i.e.\ $\mathcal{H}=\mathfrak{P}(\Omega_X)$.
\end{assumption}
Observe that Assumption~\ref{as:features} implies Assumption~\ref{as:aux} if $0 < P[X=i]$, $i = 1, \ldots, d$, 
with 
\begin{equation}\label{eq:discreteFeatures}
\varphi(i,y) = \frac{P[x=i\,|\,Y=y]}{P[X=i]}, \quad y \in \Omega_Y, i \in \Omega_X = \{1, \ldots, d\}.
\end{equation}

The combination of categorical features and labels with non-categorical values for regression problems
where the labels are to be predicted based on feature observations might appear less relevant and uncommon.
However, there are potential applications for findings in this setting:
\begin{itemize}
\item Estimating multivariate densities is harder than estimating univariate densities. 
Therefore, if the feature space is high-dimensional Euclidean and the label space is the real line it could
be computationally convenient to work with the empirical measure for the features (hence assuming the 
feature space to be finite) but to estimate a continuous (or mixed discrete-continuous) density for
the labels. 
\item In Natural Language Processing (NLP), features are categorical even if the feature space typically
is very large. For instance, Hopkins and King~\cite{hopkins2010method} describe a social 
sciences setting for class distribution estimation with digitized text as features.
\item A motor insurance company which predicts the size $Y$ of the 
claims (i.e.\ continuous labels with a point mass in $0$) incurred with a policy from 
categorical feature data $X$ of the policy holders like type of vehicle, age or region of domicile.
It appears natural to predict next year's distribution of the claim sizes  aggregated over
all policies (i.e.\ $Q_Y$) based on the assumption that the claim size (label) distribution conditional on the features
remains the same as in the current year (i.e.\ $Q_{Y|X} = P_{Y|X}$). This effectively means assuming covariate shift
in the sense of Definition~\ref{de:covariatelabelShift}~(i)
for the evolution of the joint distribution of claim size and policy features between the current
and next year.\\
However, one also could assume label shift (Definition~\ref{de:covariatelabelShift}~(ii)) instead. 
This would mean that the probabilities
of feature categories conditional on the incurred claim sizes (labels) remain constant while
the unconditional claim size (label) distribution potentially changes, i.e.\ $Q_{X|Y} = P_{X|Y}$
but potentially $Q_Y \neq P_Y$.
\item In some areas of application, discretisation of features is common. See for instance Schutte et al.~\cite{schutte2026impact}
for the area of credit risk and R\"over and Friede~\cite{Roever2017discrete} for a principled approach to discretisation.
\end{itemize}

\begin{subequations}
Under Assumption~\ref{as:features}, in the EM iteration \eqref{eq:hnCat} remains unchanged but
the number of equations is finite instead of possibly uncountably infinite for general feature spaces.
Regarding \eqref{eq:EMconsolidated}, under Assumption~\ref{as:features} there is indeed a 
computationally lighter version:
\begin{equation}\label{eq:gnCat}
g_{n+1}(y) \ = \ g_0(y) \prod_{i=0}^n \sum_{x=1}^d \frac{h(x)}{h_i(x)}\,P[X=x\,|\,Y=y], \quad y \in \Omega_Y.
\end{equation}
Note that under Assumption~\ref{as:features}, in the setting of Theorem~\ref{th:charac}~(ii) 
and Theorem~\ref{th:alternateB} the feature densities $\frac{d Q_X^\ast}{d P_X} = h$ and 
$\frac{d Q^{(n)}_X}{d P_X} = h_n$ can be represented for $x \in \Omega_X$ with $P[X=x] > 0$ as
\begin{equation}
h(x) \ = \ \frac{Q^\ast[X=x]}{P[X=x]} \quad \text{and}\quad h_n(x) \ = \ \frac{Q^{(n)}[X=x]}{P[X=x]}.
\end{equation}
\end{subequations}

The following example suggests that under Assumption~\ref{as:features} the EM algorithm in the
face of non-categorical labels and categorical features is computationally manageable at least for
small feature space sizes and parametric continuous label distributions.

\begin{table}
\begin{center}
\caption{Results for Example~\ref{ex:contLabels}.}
\label{tab:results}
\begin{tabular}{|c|c|c|c|}\hline
Initial label density & \# iterations & $Q^{(n)}[X=1]$ & Final KL\\\hline\hline
$N(0,1)$ & 42 & 0.39992 & $1.4 \times 10^{-8}$\\\hline
$N(1, 0.8)$ & 52 & 0.39989 &  $2.5 \times 10^{-8}$ \\\hline
$N(-0.5, 1.5)$ & 25 & 0.39995 & $4.8 \times 10^{-9}$ \\\hline
\end{tabular}
\end{center}  
\end{table}

\begin{example}\label{ex:contLabels}\emph{%
Let $\Omega_Y = \mathbb{R}$ and $\Omega_X = \{1,2\}$.
Define the source distribution $P$ by 
$Y \sim N(\mu,\sigma)$ for some constants $\mu \in \mathbb{R}$ and $\sigma > 0$, and
$P[X = 1\,|\,Y=y] = \left(1 + \exp\left(-\frac{y-a}{b}\right)\right)^{-1}$ for some constants $a\in\mathbb{R}$
and $b > 0$.\\
The target feature distribution is pre-specified as $Q^\ast[X=1] = q = 1 - Q^\ast[X=2]$ for some constant
$0 < q < 1$.
For the computation, the parameters for $P$ and $Q^\ast$ are fixed as follows:
\begin{equation*}
\mu = 0,\ \sigma=1; \quad a = -2,\ b = 1; \quad q = 0.4.
\end{equation*}
Observe that from this choice of parameters it follows that $P[X=1] \approx 0.15546 \neq Q^\ast[X=1]$ such
that there is some distribution shift between $P$ and $Q$. We model the distribution shift through covariate shift and
label shift and calculate the resulting label densities.\\
Under covariate shift and Assumption~\ref{as:features}, the target label density $g$ with respect to the
source label distribution is determined by
\begin{equation*}
g(y) \ =\ E_P[h(X)\,\vert\,Y=y]\ =\ \sum_{x=1}^d \frac{Q^\ast[X=x]}{P[X=x]}\,P[X=x\,\vert\,Y=y], \quad y \in \Omega_Y.
\end{equation*}
Under the label shift assumption, we run the EM algorithm described by \eqref{eq:hnCat} and \eqref{eq:gnCat} 
for three normal initial label densities $g_0$ and stop the iteration according to criterion~\eqref{eq:stop} with $\varepsilon =
10^{-8}$.\\
Table~\ref{tab:results} describes the initial label densities and lists the results of the calculations: 
Number of iterations performed (\# iterations), final approximate target feature distribution ($Q^{(n)}[X=1]$), 
final KL divergence between predefined target feature density and approximate target feature density (Final KL).\\
Figure~\ref{fig:1} shows the source label Lebesgue density, the target label Lebesgue density under covariate shift 
and the final approximate target Lebesgue label densities under label shift based on the different initial label
densities.
} \hfill \qed
\end{example}

\begin{figure}
  \centering
  \caption{Source label density vs.\ covariate shift target label density and EM target label densities with initialisations 
  according to Table~\ref{tab:results}.}\label{fig:1}
  \includegraphics[width=12cm]{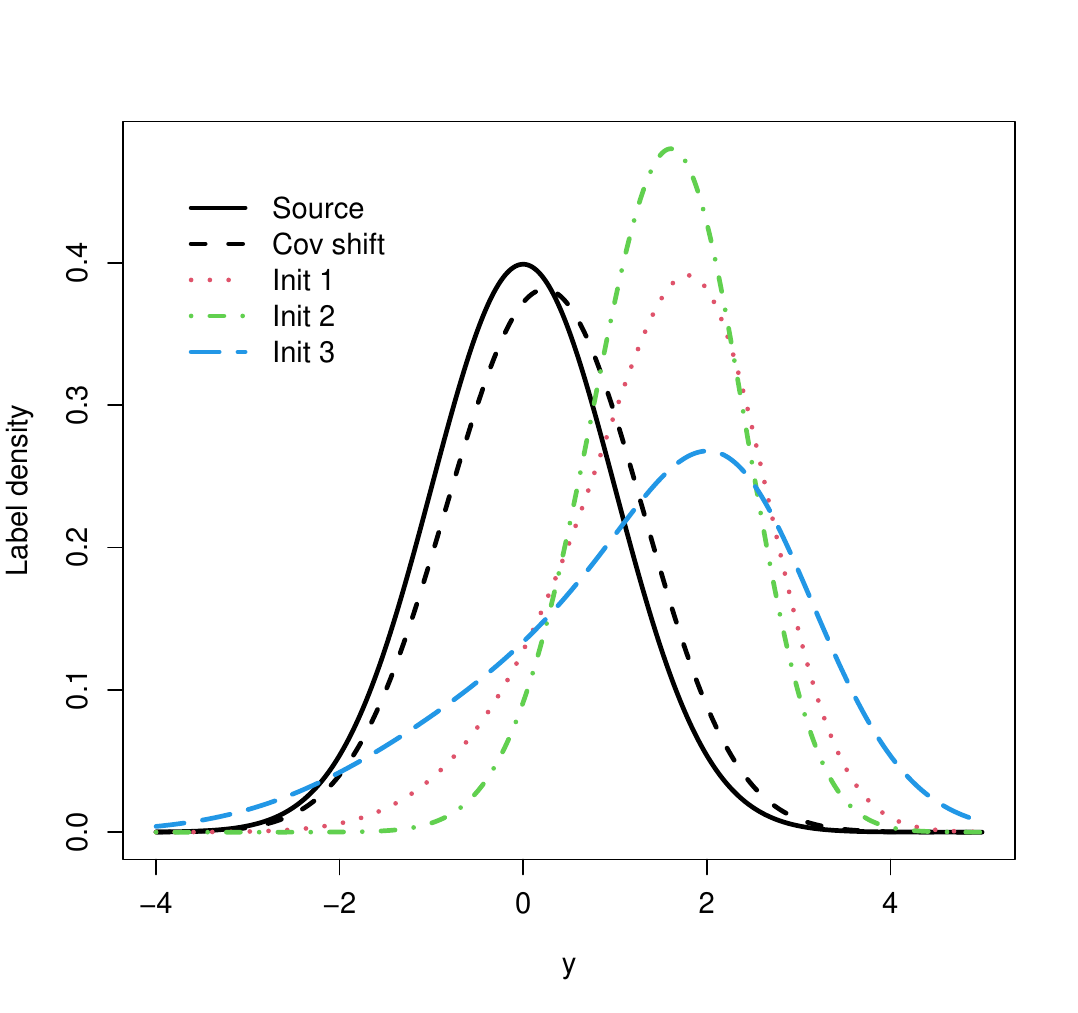}
\end{figure}

There are three main observations from Example~\ref{ex:contLabels}:
\begin{itemize}
\item[1)] At the halt of the EM algorithm, the predefined target feature distribution is almost perfectly matched
for all three different initialisations.
\item[2)] The target label density resulting from the EM algorithm depends conspicuously on the initialisation.
\item[3)]
All three target label densities resulting from the label shift assumptions clearly differ from the target label density
resulting from the covariate shift assumption.
\end{itemize}

Observation~1) is no incident but a typical outcome under Assumption~\ref{as:features}. Suppose as in
Proposition~\ref{pr:convergence} that the target label densities $g_n$ calculated with the EM algorithm
converge to a limit density $g$. Then also the target feature densities $h_n$ constructed when running
the algorithm converge to some density $\hhat$ which in general is not identical with the predefined target
feature density $h$ (cf.\ Remark~\ref{rm:achieved}). However, the following corollary to Proposition~\ref{pr:convergence}
shows that under Assumption~\ref{as:features} $h = \hhat$ is likely to occur.

\begin{corollary}\label{co:exactFit}
In the setting and with the notation of Theorem~\ref{th:charac}~(ii), 
Theorem~\ref{th:alternateB} and Proposition~\ref{pr:convergence}, 
suppose in addition that Assumption~\ref{as:features} holds.
If the source posterior probabilities $y \mapsto P[X=1\,|\,Y=y], \ldots, y \mapsto P[X=d\,|\,Y=y]$ are 
linearly independent in the space of $P_Y$-integrable random variables, then it follows that
$Q^\ast[X=x] = \widehat{Q}[X=x]$ for all $x \in \Omega_X = \{1, \ldots, d\}$.
\end{corollary}

\begin{proof}[Proof of Corollary~\ref{co:exactFit}.]
Close inspection of the proof of Proposition~\ref{pr:convergence} reveals that under Assumption~\ref{as:features}
in Proposition~\ref{pr:convergence}~(iv) `$\le$' can be replaced with `$=$', i.e.\ it holds that
\begin{equation*}
E_P\left[\frac{h(X)}{\hhat(X)}\,\Big|\,\sigma(Y)\right] \ = \ 1.
\end{equation*}
Taking into account that $h(x) = \frac{Q^\ast[X=x]}{P[X=x]}$ and $\hhat(x) = \frac{\widehat{Q}[X=x]}{P[X=x]}$ 
under Assumption~\ref{as:features} implies
\begin{equation*}
\sum_{x=1}^d \left(\frac{Q^\ast[X=x]}{\widehat{Q}[X=x]} -1\right) P[X=x\,|\,Y=y] \ = \ 0, 
\quad P_Y\text{-a.s.\ for all}\ y \in \Omega_Y.
\end{equation*}
The assertion now follows from the linear independence of $P[X=1\,|\,Y=\cdot], \ldots,$ $P[X=d\,|\,Y=\cdot]$.
\end{proof}

Linear independence of $P[X=1\,|\,Y=\cdot]$ and $P[X=2\,|\,Y=\cdot]$ is obvious in Example~\ref{ex:contLabels}.
But it also is likely to occur if the cardinality of the label space $\Omega_Y$ (e.g.\ $\Omega_Y = \mathbb{R}$)
is much greater than the cardinality $d$ of the feature space $\Omega_X$ and the posterior probabilities are not
`flat', i.e.\ the dependence between the features and the labels under the source distribution is not too weak. 

We may draw the following conclusions from Example~\ref{ex:contLabels} for the use of the 
EM algorithm for label distribution estimation under 
Assumption~\ref{as:features} of categorical features in the face of label shift:
\begin{itemize}
\item The EM algorithm solves the label distribution estimation problem under label shift for real-valued labels, 
even with exact fit of the predefined target feature distribution.
\item However, the resulting target label density depends on the initial label density at the start of the
algorithm. As all initial densities entail exact fit of the target feature distribution there is no
obvious criterion that would make any initial density preferable.
\item Nonetheless, the source label density is an obvious choice for the initial density -- which in some
cases might suffice for justification of its selection as initial density.
\end{itemize}

\section{Generalized label shift (GLS)}
\label{se:GLS}

A popular approach to tackling distribution shift (i.e.~domain adaptation) is representation learning.
The aim with representation learning is to find a mapping of the features into another feature space such
that for the joint distribution of the mapped features and the labels the distribution shift is removed or
reduced to a more manageable type of shift. Tachet des Combes et al.~\cite{tachetdescombes2020domain} proposed 
with `generalized label shift' (GLS) a variant of representation learning where the representation of the
features results in label shift between the source and the target distributions. Wu~\cite{wu2025prominent}
use the term `Conditionally invariant components' for GLS and suggest an approach
to plausibly identifying such components of the feature variable 
in the absence of label observations in the target dataset.

He et al.~\cite{he2022domain}  and Tasche~\cite{tasche2022factorizable} (Proposition~4)
noted that GLS implies FJS if the representation mapping is sufficient for the features,
i.e.\ under both the source and the target distribution the label distribution conditional on the feature 
variable equals the label distribution conditional on the mapped features. The following proposition shows
for a sufficient representation mapping that also FJS between the source and the target joint distributions
of the mapped features and the labels implies FJS between the original source and target distributions.
The function $R$ appearing in Proposition~\ref{pr:GLS} stands for the representation mapping. Typically,
$R$ would be some function applied to the feature variable $X$, i.e.\ $R = T(X)$ for some measurable function $T$.

\begin{proposition}\label{pr:GLS}
In the setting of Section~\ref{se:setting} and under Assumption~\ref{as:main} with $\frac{d Q}{d P}=f=f(X,Y)$, 
let $R$ be a measurable function on $(\Omega, \mathcal{F})$ such that
$\sigma(R) \subset \sigma(X)$ holds. Assume that $P|\sigma(R,Y)$ and $Q|\sigma(R,Y)$
are related through FJS in the sense of Definition~\ref{de:FJS} with 
$\frac{d Q|\sigma(R,Y)}{d P|\sigma(R,Y)} = \hbar(R)\,\gbar(Y)$ for measurable
non-negative functions $\hbar$ and $\gbar$. Suppose furthermore that it holds that
\begin{equation}\label{eq:sufficiency}
\begin{split}
P[Y\in G\,|\,\sigma(X)] & = P[Y\in G\,|\,\sigma(R)], \ \text{and} \\
Q[Y\in G\,|\,\sigma(X)] & = Q[Y\in G\,|\,\sigma(R)]
\end{split}
\end{equation}
for all $G \in \mathcal{G}$ (i.e.\ $R$ is a sufficient dimension reduction for $X$ with respect to $Y$ under both
$P$ and $Q$ in the sense of Definition~1.1 of Adragni and Cook~\cite{adragni2009sufficient}). 
In addition, assume that there exists a regular
conditional distribution $P_{Y|R}$	 of $Y$ with respect to $R$ under $P$ 
in the sense of Definition~\ref{de:regCondDist}. 
Then also $P$ and $Q$ are related through
FJS and it holds that
\begin{equation}\label{eq:alsoGLS}
f(X,Y) \ = \ \frac{E_P[f(X,Y)\,|\,\sigma(X)]}{E_P[\gbar(Y)\,|\,\sigma(X)]}\, \gbar(Y).
\end{equation}
\end{proposition}
See Appendix~\ref{se:proofs} for a proof of Proposition~\ref{pr:GLS}.
Recall that if $P|\sigma(R,Y)$ and $Q|\sigma(R,Y)$ are related through label shift, then according to Definition~3.1 of
Tachet des Combes et al.~\cite{tachetdescombes2020domain} $P$ and $Q$ are related through 
\emph{generalized label shift (GLS)}.

\begin{corollary}\label{co:GLS}
Under the assumptions of Proposition~\ref{pr:GLS}, suppose that 
$P|\sigma(R,Y)$ and $Q|\sigma(R,Y)$ are related through
label shift in the sense of Definition~\ref{de:covariatelabelShift} (ii). Denote by $g = \frac{d Q_Y}{d P_Y}$
the label density of $Q_Y$ with respect to $P_Y$. Then $P$ and $Q$ are related through
FJS and it holds that 
\begin{equation}\label{eq:coGLS}
f(X,Y)\ = \ \frac{E_P[f(X,Y)\,|\,\sigma(X)]}{E_P[g(Y)\,|\,\sigma(X)]}\, g(Y).
\end{equation}
\end{corollary}

\begin{proof}[Proof of Corollary~\ref{co:GLS}.]
Corollary~\ref{co:GLS} is a special case of Proposition~\ref{pr:GLS} since 
Proposition~\ref{pr:SpecCases} and its proof imply that 
$\frac{d Q|\sigma(R,Y)}{d P|\sigma(R,Y)} = g(Y)$. Replacing $\gbar$ in 
\eqref{eq:alsoGLS} with $g$ yields \eqref{eq:coGLS}.
\end{proof}

Corollary~\ref{co:GLS} suggests an alternative to domain adaptation through representation learning in combination 
with a GLS assumption. One of the objectives of representation learning is avoiding losing predictive power 
for classification or regression in comparison to classification or regression based on the original features.
For this reason, the sufficiency condition \eqref{eq:sufficiency} of Proposition~\ref{pr:GLS} is 
likely to be satisfied at least approximately.

But then by Corollary~\ref{co:GLS}, if $Q$ and $P$ are related through GLS they also are related through
FJS, and even through constrained FJS as discussed in Remark~\ref{rm:achieved} above.
This observation applies independently of the representation mapping $R$. Hence it might make sense to renounce finding
a representation $R$ for achieving GLS and instead resort to the EM algorithm for solving Problem~\ref{prB} as presented in
Section~\ref{se:probB}. 

\section{Concluding remarks}
\label{se:conclusions}

In this paper, we have extended the notion of factori\-zable joint shift (FJS) to cover the regression 
case with general label space and shed more light on the relation between FJS, covariate shift and label shift.
We also have revisited the property of generalized label shift (GLS) to imply FJS if the involved representation
mapping is sufficient for the original features. We have shown that this property is preserved in 
settings with general label spaces.

With regard to fitting FJS to given target feature and label marginal distributions, in Section~\ref{se:probA} 
we have presented an approach based on a non-linear integral equation for which we have proposed an iterative solution. 
Regarding
fitting FJS to a given target feature distribution without recourse to any label observations, 
in Section~\ref{se:probB} we have
presented a generalisation to label distribution estimation of the expectation maximisation (EM) approach for 
estimating target prior class probabilities under label shift
(Saerens et al.~\cite{saerens2002adjusting}) and shown that it is fit for purpose.

Future research topics could be the computational and convergence properties of the estimation approaches proposed 
in Sections~\ref{se:probA} and \ref{se:probB} as well as
the study of other specific types of distribution shift
(e.g.~sparse joint shift, Chen et al.~\cite{chen&zaharia&Zou:SJS}) in the face of general label spaces.

\section*{Acknowledgments}
The author thanks an anonymous reviewer for suggestions
that helped to significantly improve an earlier version of this article.

\appendix

\section{Notation and technical terms}
\label{se:notation}

Assume the setting of Section~\ref{se:setting} above.

\indent\textbf{Expected values.} For $\mathcal{F}$-Borel-measurable random variables 
$T: \Omega \to \mathbb{R}, \omega \mapsto T(\omega)$,  the \emph{expected value of $T$ with respect to $P$} is
\begin{equation*}
E_P[ T ] \ = \ \int_\Omega T(\omega)\,P(d \omega) \ = \ \int T\,d P
\end{equation*}
if $\int \lvert T\rvert \,dP < \infty$ or $T$ is non-negative. In the same way, expected values are defined
with respect to other probability spaces like $(\Omega_X, \mathcal{H}, P_X)$, $(\Omega, \mathcal{F}, Q)$ etc.

For example, for $\mathcal{H}$-Borel-measurable  $T: \Omega_X \to \mathbb{R}, x \mapsto T(x)$,  
the expected value of $T$ with respect to $P_X$ is
\begin{equation*}
E_{P_X}[ T ] \ = \ \int_{\Omega_X} T(x)\,P_X(d x) \ = \ \int T\,d P_X \ =\ E_P[T(X)]
\end{equation*}
if $\int \lvert T\rvert \,dP_X < \infty$ or $T$ is non-negative. 

\textbf{Indicator function.} For a set $M$, the \emph{indicator function} $\mathbf{1}_M$ is defined as 
$\mathbf{1}_M (m) = 1$ if $m \in M$ and $\mathbf{1}_M (m) = 0$ if $m \notin M$.

\textbf{`Almost surely'.} We use the shorthand `$\PP$-a.s.' for `$\PP$-almost surely' which is a
shorter version of the phrase `with probability $1$ under the distribution~$\PP$'. Any of the  
probability measures encountered in the following can be substituted for $\PP$.

The following notation for conditional expected values, probabilities and distributions is intended to
pin down a rigorous understanding of these concepts because sometimes they are only superficially or not at all
explained in the machine learning literature 
(e.g.~ Hastie et al.~\cite{hastie2009elements}, Murphy~\cite{murphy2012machine}, Wu~\cite{Wu2026trustworthy}).

\textbf{Expected values conditional on $\sigma$-algebras.} 
For $\mathcal{F}$-Borel-measurable $T: \Omega \to \mathbb{R},
\omega \mapsto T(\omega)$ and any sub-$\sigma$-algebra $\mathcal{C} \subset\mathcal{F}$, the \emph{expected value of $T$
conditional on $\mathcal{C}$ under $P$} is the $\mathcal{C}$-Borel-measurable 
random variable $E_P[T\,|\,\mathcal{C}]$ on $(\Omega, \mathcal{F}, P)$ which is determined $P$-a.s.\ by the property
\begin{equation*}
E_P\bigl[E_P[T\,|\,\mathcal{C}]\,\mathbf{1}_C\bigr]\ = \ E_P[T\,\mathbf{1}_C], \quad \text{for all}\ C \in \mathcal{C},
\end{equation*}
if $\int \lvert T\rvert \,dP < \infty$ or $T$ is non-negative.

For $F \in \mathcal{F}$, the  
\emph{probability of $F$ conditional on $\mathcal{C}$ under $P$} is defined by 
$P[F\,|\,\mathcal{C}]  =  E_P[\mathbf{1}_F\,|\,\mathcal{C}]$.
Expected values and probabilities conditional on $\sigma$-algebras under other 
probability measures like the target distribution
$Q$ are defined analogously.

\textbf{Conditional expected values of densities.} 
We also make use of the notation $P|\mathcal{G}$. While the probability measure
$P$ on the measurable space $(\Omega, \mathcal{F})$ is a mapping $P: \mathcal{F} \to [0,1]$,
$F \mapsto P[F]$, the term $P|\mathcal{G}$ with $\mathcal{G} \subset \mathcal{F}$ stands for the restricted mapping 
$P: \mathcal{G} \to [0,1]$. This notation is mainly of interest in the context of the following
useful fact.
\begin{lemma}\label{le:density}
Let $(\Omega, \mathcal{F})$ be a measurable space. Assume that $P$ and $Q$ are probability measures
on $(\Omega, \mathcal{F})$ and $\mathcal{G}\subset \mathcal{F}$ is a sub-$\sigma$-algebra of $\mathcal{F}$. 
If $f = \frac{d Q}{d P}$ is a density of $Q$ with respect to $P$ on $\mathcal{F}$ then 
$g = E_P[f\,|\,\mathcal{G}]$ is a density of $Q$ with respect to $P$ on $\mathcal{G}$, i.e.\
it holds that $g = \frac{d Q|\mathcal{G}}{d P|\mathcal{G}}$.
\end{lemma}

\textbf{Expected values conditional on random variables.} Let $T: \Omega \to \mathbb{R}, \omega \mapsto T(\omega)$
be $\mathcal{F}$-Borel-measurable and non-negative or satisfy $\int \lvert T\rvert \,dP < \infty$ and 
let $V$ be a measurable function on
$(\Omega, \mathcal{F})$ with values in a measurable space $(\Omega_V, \mathcal{V})$. Then the 
$\mathcal{V}$-Borel-measurable real-valued function $E_P[T\,|V = \cdot]$ is the 
\emph{expected value of $T$ conditional on $V$ under $P$} if it holds for all $M \in \mathcal{V}$ that
\begin{equation}\label{eq:condV}
E_P[T\,\mathbf{1}_{\{V \in M\}}]\ =\ \int_M E_P[T\,|V = v] P_V(dv),
\end{equation}
with $\{V \in M\} = V^{-1}(M)$ and the shorthand notation $E_P[T\,|V = v] = E_P[T\,|V = \cdot](v)$ 
for $v \in \Omega_V$. Note that $E_P[T\,|V = \cdot]$ is determined $P_V$-a.s.\ by property~\eqref{eq:condV}.

For $F \in \mathcal{F}$, the  
\emph{probability of $F$ conditional on $V$ under $P$} is defined by 
$P[F\,|\,V = \cdot]  =  E_P[\mathbf{1}_F\,|\,V = \cdot]$.
Expected values and probabilities conditional on random variables under other 
probability measures like the target distribution $Q$ are defined analogously.

The definitions of expected values conditional on $\sigma$-algebras and on random variables respectively 
are related through the following identity:
\begin{equation}\label{eq:identity}
E_P[T\,|\,\sigma(V)](\omega) \ =\ E_P[T\,|V = V(\omega)] \quad P\text{-a.s.\ for}\ \omega \in \Omega,
\end{equation}
where $\sigma(V) = V^{-1}(\mathcal{V})$ is the sub-$\sigma$-algebra of $\mathcal{F}$ generated by $V$.

Expected values and probabilities conditional on $\sigma$-algebras or on random variables are well-defined for
all combinations of probability spaces and integrable or non-negative random variables whose conditional expected values 
are meant to be determined
(Bauer~\cite{Bauer1981ProbTheory}, Sections~10.1 and 10.2; Klenke~\cite{klenke2013probability}, Sections~8.2 and 8.3). 
This is not necessarily true for conditional distributions as in the following definition.

\begin{definition}[Regular conditional distribution]\label{de:regCondDist}
In the setting of Section~\ref{se:setting}, a mapping $K: \Omega_X \times \mathcal{G} \to [0,\infty)$ is
called (regular) \emph{conditional distribution of $Y$ given $X$ under $P$} if it has the following three
properties:
\begin{enumerate}
\item[(i)] For fixed $G \in \mathcal{G}$, the function $K(\cdot, G):\Omega_X \to [0,\infty), x \mapsto K(x,G)$ is
	$\mathcal{H}$-Borel-measurable.
\item[(ii)] For fixed $x \in \Omega_X$, the function $K(x, \cdot):\mathcal{G} \to [0,\infty), G \mapsto K(x,G)$ is a
	probability measure on $(\Omega_Y, \mathcal{G})$.
\item[(iii)] For all $G \in \mathcal{G}$, it holds that $P[Y\in G\,|\,X=x] = K(x, G)$
	$P_X$-a.s.~for  $x \in \Omega_X$.
\end{enumerate}
In the following, we use the notation $P_{Y|X}$ for mappings $K$ satisfying (i) to (iii) and
write $K(x, G) = P_{Y|X=x}[G]$ to make it clear that the conditional probability from (iii) is
represented with a conditional distribution.
\end{definition}

Regular conditional distributions $P_{X|Y}$ also are defined by (i) to (iii), with swapped roles of
$X$ and $Y$. Conditional distributions under the target distribution $Q$ are defined by (i) to (iii),
with $Q$ substituted for $P$.

In general, the existence of regular conditional distributions can be guaranteed only if certain conditions
on the image of the conditioned variable -- $(\Omega_Y, \mathcal{G})$ in Definition~\ref{de:regCondDist} -- are satisfied 
(Bauer~\cite{Bauer1981ProbTheory}, Section~10.3; Klenke~\cite{klenke2013probability}, Section~8.3). 
If however the label variable $Y$ is real-valued, the regular conditional distribution $P_{Y|X}$
exists, without any restrictions on the feature variable $X$  (Klenke~\cite{klenke2013probability}, Theorem~8.28).

In Section~\ref{se:source} below, we introduce Assumption~\ref{as:aux} that suffices to guarantee the
existence of both $P_{Y|X}$ and $P_{X|Y}$ and to some extent can be checked empirically.
The need to have regular conditional distributions is context-dependent. In the following, it is
mentioned explicitly when regular conditional distributions $P_{X|Y}$ or $P_{Y|X}$ in
the sense of Definition~\ref{de:regCondDist} are supposed to exist.

\begin{subequations}
Regular conditional distributions are required in particular if integration with respect to 
$P_{X|Y}$ or $P_{Y|X}$ is involved, as 
for Fubini's theorem for conditional distributions (Klenke~\cite{klenke2013probability}, Theorem~14.26):
\begin{equation}\label{eq:Fubini}
\begin{split}
E_P[T(X,Y)] & \ = \ \int \left(\int T(x,y) P_{X|Y=y}(dx)\right) P_Y(dy) \\
& \ = \	\int \left(\int T(x,y) P_{Y|X=x}(dy)\right) P_X(dx), 
\end{split}
\end{equation}
for all $\mathcal{F}$-Borel measurable $P$-integrable or non-negative functions $T=T(X,Y): \Omega \to \mathbb{R}$
if the
related regular conditional distributions $P_{X|Y}$ and $P_{Y|X}$ respectively exist.

As a consequence of \eqref{eq:Fubini} expected values conditional on the features $X$ or 
on the labels $Y$ are related to conditional distributions as follows
\begin{equation}\label{eq:condExp}
\begin{split}
E_P[T(X,Y)\,|\,Y=y] & \ = \ \int T(x,y) P_{X|Y=y}(dx), \ P_Y\text{-a.s.\ for all}\ y \in \Omega_Y,\\
E_P[T(X,Y)\,|\,X=x] & \ = \ \int T(x,y) P_{Y|X=x}(dy),\ P_X\text{-a.s.\ for all}\ x \in \Omega_X,
\end{split}
\end{equation}
under the same conditions as for \eqref{eq:Fubini}.
\end{subequations}

\section{Specific assumptions for faciliting calculations}
\label{se:facilitating}

This appendix provides assumptions that refine Assumpion~\ref{as:main} above and allow
for numerical evaluations of the formulae presented in the main part of the paper. 
The additional assumptions also facilitate reconciliation with existing work like Saerens et
al.~\cite{saerens2002adjusting} and Tasche~\cite{tasche2022factorizable}.

\subsection{Assumptions on the source distribution and conditional densities}
\label{se:source}

With additional assumptions on the properties of the source distribution $P$, the density $f$ according 
to Assumption~\ref{as:main} can be decomposed into `building blocks' which allow for interpretations 
in terms of causality or chronological order and facilitate the demarcation of specific 
types of distribution shifts.

\begin{assumption}[Regular conditional distributions]\label{as:condDist}
The following two assumptions relating to the source distribution $P$ 
are to be understood in the setting of Section~\ref{se:setting}: 
\begin{itemize}
\item[(i)] There exists a conditional distribution $P_{Y|X}$ of $Y$ with respect to $X$ in the sense of
Definition~\ref{de:regCondDist}.
\item[(ii)] There exists a conditional distribution $P_{X|Y}$ of $X$ with respect to $Y$ in the sense of
Definition~\ref{de:regCondDist}.
\end{itemize}
\end{assumption}

Thanks to \eqref{eq:condExp} Assumption~\ref{as:condDist} immediately implies the following
version of Proposition~\ref{pr:genProperties}.
\begin{corollary}\label{co:genProperties}
In the setting of Section~\ref{se:setting}, let $T = T(X,Y): \Omega \to \mathbb{R}$ 
be $\mathcal{F}$-Borel-measurable with $E_P[|T|] < \infty$ or
$T \ge 0$. 
\begin{itemize}
\item[(i)] Under Assumptions~\ref{as:main} and \ref{as:condDist}~(i), 
the density $h$ of $Q_X$  with respect to $P_X$ can be represented as
\begin{align*}
h(x) & \ = \ \int f(x,y)\,P_{Y|X=x}(dy), \quad P_X\text{-a.s.\ for}\ x \in \Omega_X,\\
\intertext{and it holds that}
E_Q[T\,|\,X=x] & \ = \ \frac{\int T(x,y)\,f(x,y)\,P_{Y|X=x}(dy)}{h(x)}, 
\end{align*}
for $x\in \Omega_X$ with $h(x) > 0$, i.e.\ with probability~$1$ under $Q_X$.
\item[(ii)]  Under Assumptions~\ref{as:main} and \ref{as:condDist}~(ii), the density $g$ of $Q_Y$ 
with respect to $P_Y$ can be represented as
\begin{align*}
g(y) & \ = \ \int f(x,y)\,P_{X|Y=y}(dx), \quad P_Y\text{-a.s.\ for}\ y \in \Omega_Y,\\
\intertext{and it holds that}
E_Q[T\,|\,Y=y] & \ = \ \frac{\int T(x,y)\,f(x,y)\,P_{X|Y=y}(dx)}{g(y)},
\end{align*}
for $y\in \Omega_Y$ with $g(y) > 0$, i.e.\ with probability~$1$ under $Q_Y$.
\end{itemize} 
\end{corollary}

Assumption~\ref{as:condDist} is also useful because of the following result that clarifies the concept of
conditional density as considered in the distribution shift context of this paper.
\begin{theorem}\label{th:condDens}\ 
\begin{itemize}
\item[(i)] Under Assumptions~\ref{as:main} and \ref{as:condDist}~(i), let $h$ be a density of $Q_X$ with
respect to $P_X$ and define
$q_{Y|X}: \Omega_X \times \Omega_Y \to [0, \infty)$ by
$q_{Y|X=x}(y) = \begin{cases}
\frac{f(x,y)}{h(x)}, & h(x) > 0\\
0, & h(x) = 0
\end{cases}$ for $x\in \Omega_X$ and $y \in \Omega_Y$.
Then for $x\in \Omega$ with $h(x) > 0$ (and therefore $Q_X$-a.s.) $q_{Y|X=x}$ is a density of 
$Q_{Y|X=x}$ with respect to $P_{Y|X=x}$.
\item[(ii)] Under Assumptions~\ref{as:main} and \ref{as:condDist}~(ii), let $g$ be a density of $Q_Y$ with
respect to $P_Y$  and define
$q_{X|Y}: \Omega_X \times \Omega_Y \to [0, \infty)$ by
$q_{X|Y=y}(x) = \begin{cases}
\frac{f(x,y)}{g(y)}, & g(y) > 0\\
0, & g(y) = 0
\end{cases}$ for $x\in \Omega_X$ and $y \in \Omega_Y$.
Then for $y\in \Omega$ with $g(y) > 0$ (and therefore $Q_Y$-a.s.) $q_{X|Y=y}$ is a density of 
$Q_{X|Y=y}$ with respect to $P_{X|Y=y}$.
\end{itemize}
\end{theorem}

See Appendix~\ref{se:proofs} for a proof of Theorem~\ref{th:condDens}.

In this paper, the densities $q_{X|Y}$ and $q_{Y|X}$ as defined in Theorem~\ref{th:condDens} both are 
called \emph{conditional density}. 
Based on Thorem~\ref{th:condDens}, versions for general label spaces can be formulated of 
the `normal form' of the joint density $f$  (Tasche~\cite{tasche2022factorizable}, Theorem~1),
of the feature density $h$ (Tasche~\cite{tasche2022factorizable}, Corollary~1) and
of the `posterior correction formula'\footnote{%
Called `Bayes update rule' in Donyavi et al.~\cite{DonyaviMatch2025}.} 
(Tasche~\cite{tasche2022factorizable}, Corollary~2).

\begin{samepage}
\begin{subequations}
\begin{corollary}\label{co:main}\
\begin{itemize}
\item[(i)] Under Assumption~\ref{as:main} and Assumption~\ref{as:condDist}~(ii),
the density $f$  of the joint target
distribution $Q$ of $X$ and $Y$ with respect to the joint source distribution $P$ can be represented
in terms of the conditional density $q_{X|Y}$ as defined in Thorem~\ref{th:condDens} and the label density $g$  as 
\begin{equation}\label{eq:genDens}
f(x,y) \ = \ q_{X|Y=y}(x)\,g(y), \quad \text{for}\ (x,y) \in \Omega.
\end{equation}
\item[(ii)] Under Assumption~\ref{as:main} and Assumptions~\ref{as:condDist}~(i) and (ii),
the feature density $h$ of $Q_X$ with respect to $P_X$ can be written 
in terms of the conditional density $q_{X|Y}$, the label density $g$
and the source conditional distribution $P_{Y|X}$  as
\begin{equation}\label{eq:hDens}
h(x) \ = \ \int_{\Omega_Y} q_{X|Y=y}(x)\,g(y)\,P_{Y|X=x}(dy), \quad \text{for}\ x \in \Omega_X.
\end{equation}
\item[(iii)] Under Assumption~\ref{as:main} and Assumptions~\ref{as:condDist}~(i) and (ii),
the posterior probabilities of $Y$ given $X$ under $Q$ can
be described in terms of the conditional density $q_{X|Y}$, the label density $g$
and the source conditional distribution $P_{Y|X}$ 
for all $x \in \Omega_X$ with $h(x) > 0$ (i.e.\ $Q_x$-a.s.) by
\begin{equation}\label{eq:correction}
Q_{Y|X=x}[G] \ = \ \frac{\int_G q_{X|Y=y}(x)\,g(y)\,P_{Y|X=x}(dy)}{h(x)}, \quad\text{for}\ G \in \mathcal{G}.
\end{equation}
\end{itemize}
\end{corollary}
\end{subequations}
\end{samepage}

\begin{proof}[Proof of Corollary~\ref{co:main}.]
(i) is basically the definition of $q_{X|Y=y}(x)$ in Theorem \ref{th:condDens} (ii). (ii) is the combination
of (i) and Corollary~\ref{co:genProperties}~(i). (iii) is implied also by (i) and Corollary~\ref{co:genProperties}~(i).
\end{proof}

\subsection{Facilitating calculations by further assumptions}
\label{se:further}

The setting of Tasche \cite{tasche2022factorizable} differs from the setting in Section~\ref{se:setting}
through the additional assumption of the label variable $Y$ being categorical, i.e.\ when
the context is classification. This  is stated formally
in the following assumption where without loss of generality the categories are supposed to be natural numbers.
\begin{assumption}[Categorical label space]\label{as:discrete}
In the setting of Section~\ref{se:setting}, the label space $\Omega_Y$ is finite
with $\Omega_Y = \{1, \ldots, d\}$, $d \ge 2$. The $\sigma$-algebra
$\mathcal{G}$ is the power set of $\Omega_Y$, i.e.\ $\mathcal{G}=\mathfrak{P}(\Omega_Y)$.
\end{assumption}

Under Assumption~\ref{as:discrete},  suppose in addition $0 < P[Y=i]$, 
$i = 1, \ldots, d$. Then thanks to the fact that $\Omega_Y$ is finite, Assumptions~\ref{as:condDist}~(i)
and (ii) both hold true. 

Under distribution shift as expressed through Assumption~\ref{as:main}, 
to reconcile Corollary~\ref{co:main} with the corresponding statements of
Tasche~\cite{tasche2022factorizable} under Assumption~\ref{as:discrete}, 
observe that $q_{X|Y=i} = h_i$ for the conditional densities 
defined in  Lemma~1 of Tasche~\cite{tasche2022factorizable} and $g(i) = \frac{Q[Y=i]}{P[Y=i]}$ for $i=1, \ldots, d$. 
Then it follows from Corollary~\ref{co:main} (iii) with $G = \{i\}$ that
\begin{equation}\label{eq:discrete.correction}
Q[Y=i\,|\,X=x] \ =\ \frac{h_i(x)\,\frac{Q[Y=i]}{P[Y=i]}\,P[Y=i\,|\,X=x]}
{\sum_{j=1}^d h_j(x)\,\frac{Q[Y=j]}{P[Y=j]}\,P[Y=j\,|\,X=x]}, \quad i=1, \ldots, d.
\end{equation}
\eqref{eq:discrete.correction} restates Corollary~2 of Tasche~\cite{tasche2022factorizable} 
in the notation of this paper. Similarly, Theorem~1 and
Corollary~1 of Tasche~\cite{tasche2022factorizable}  are implied by Corollary~\ref{co:main}~(i) and (ii) 
respectively.

In the case where the label variable $Y$ is not categorical, 
even if Assumption~\ref{as:condDist}~(i) or Assumption~\ref{as:condDist}~(ii) can be shown to be
satisfied thanks to certain properties of the feature or label spaces (like in the case where $Y$ is
real-valued, cf.~the comments after Definition~\ref{de:regCondDist}), 
it is not yet clear how $P_{Y|X}$ and $P_{X|Y}$ can be determined.
For this reason the subsequent Assumption~\ref{as:aux} is useful 
in so far as it implies both Assumption~\ref{as:condDist}~(i) 
and Assumption~\ref{as:condDist}~(ii) as well as convenient representations of $P_{Y|X}$ and $P_{X|Y}$
in Propostion~\ref{pr:condDist} below.

Define in the setting of Section~\ref{se:setting} the probability measure $P_X \otimes P_Y$ on 
$(\Omega, \mathcal{F})$ as the product measure of $P_X$ and $P_Y$.
Then $P_X \otimes P_Y$ is uniquely determined as the distribution on $(\Omega, \mathcal{F})$
such that $X$ and $Y$ are independent, i.e.~by the property
\begin{equation*}
P_X \otimes P_Y[H \times G] = P_X \otimes P_Y[X \in H, Y \in G] = P_X[H]\,P_Y[G], \quad H \in \mathcal{H},
	G \in \mathcal{G}.
\end{equation*}
In general, it holds that $P_X \otimes P_Y \neq P$ but the marginal distributions
of $X$ and $Y$ under $P_X \otimes P_Y$ and $P$ are equal.

\begin{assumption}[Absolute continuity of source distribution]\label{as:aux}
In the setting of Section~\ref{se:setting}, the source distribution $P$ is absolutely continuous 
on $(\Omega, \mathcal{F})$ with
respect to the product measure $P_X \otimes P_Y$ of the feature source distribution $P_X$ and
the label source distribution $P_Y$. Denote the density of $P$ 
with respect to $P_X \otimes P_Y$ by $\varphi$, i.e.~it holds that\footnote{%
$\varphi$ serves as a representation of the dependence structure of $X$ and $Y$ under $P$ in a 
similar way as a copula expresses the dependence structure of a random vector in a Euclidean space 
(cf.~Durante and Sempi~\cite{durante2016principles}, Theorem~2.1.1 and Theorem~2.2.1 (Sklar's theorem)).} 
\begin{equation}\label{eq:phi}
\frac{d P}{d P_X \otimes P_Y}\ =\ \varphi.
\end{equation}
\end{assumption}

Observe that Assumption~\ref{as:discrete} implies Assumption~\ref{as:aux} if $0 < P[Y=i]$, $i = 1, \ldots, d$, 
with 
\begin{equation}\label{eq:discreteDensity}
\varphi(x,i) = \frac{P[Y=i\,|\,X=x]}{P[Y=i]}, \quad x \in \Omega_X, i \in \Omega_Y = \{1, \ldots, d\}.
\end{equation}

The choice of $P_X \otimes P_Y$ as reference measure in Assumption~\ref{as:aux} is convenient 
because no additional assumptions
on the structural properties of the instance space $(\Omega, \mathcal{F})$ are required, in contrast e.g.~to
the case where the Lebesgue measure serves as reference measure.
Nonetheless, having a density of $P$ with respect to $P_X \otimes P_Y$ does not exclude representations of $P$
with other densities for computational purposes. For instance, if $P_X$ and $P_Y$ have both Lebesgue densities,
then by the density chain rule $P$ also has a Lebesgue density.

Intuitively, Assumption~\ref{as:aux} is true if the supports of $P$ and $P_X \otimes P_Y$ have the same dimension.
As Filipovi{\'c} and Schneider~\cite{filipovic2025kernel} point out in Section~5.2.2, 
Assumption~\ref{as:aux} is violated in particular if $X$ is a function of $Y$ or vice versa.
Filipovi{\'c} and Schneider~\cite{filipovic2025kernel} (Lemma~5.3) 
provide a sufficient criterion for Assumption~\ref{as:aux} to be true.
Moreover, there are feasible methods for finding the density
of $P$ with respect to $P_X \otimes P_Y$.
Filipovi{\'c} and Schneider~\cite{filipovic2025kernel} (Section~5.2) propose to use a 
reproducing kernel Hilbert space (RKHS) method for estimating $\varphi$. 
 They also comment on other methods for doing so from the previous literature on the
subject. In addition, $\varphi$ can be determined by means of a binary probabilistic classifier, see 
Appendix~\ref{se:DensEst} below. An advantage of the classification approach is that in principle 
the spaces $(\Omega_X, \mathcal{H})$ and $(\Omega_Y, \mathcal{G})$ need not satisfy any specific conditions
for the approach to work.

\begin{proposition}\label{pr:condDist}
Assumption~\ref{as:aux} implies both Assumption~\ref{as:condDist}~(i) and Assumption~\ref{as:condDist}~(ii), with
\begin{equation*}
\begin{split}
P_{Y|X=x}[G] & = \int_G \varphi(x,y)\,P_Y(dy),\quad \text{for}\ x \in \Omega_X, G \in \mathcal{G},\quad \text{and}\\
P_{X|Y=y}[H] & = \int_H \varphi(x,y)\,P_X(dx), \quad \text{for}\ y \in \Omega_Y, H \in \mathcal{H}.
\end{split}
\end{equation*}
\end{proposition}

\begin{proof}[\textbf{Proof of Proposition~\ref{pr:condDist}.}]
Actually, Proposition~\ref{pr:condDist} is a special case of Theorem~\ref{th:condDens}. For instance, to prove
the statement regarding Assumption~\ref{as:condDist}~(i), substitute $P$ for $Q$, $P_X \otimes P_Y$ 
for $P$ and $\varphi$ for $f$ in Theorem~\ref{th:condDens}. The theorem is applicable because under 
$P_X \otimes P_Y$ the marginal distributions of $X$ and $Y$ are regular conditional distributions of $X$ given $Y$
and $Y$ given $X$ respectively. Moreover, since the marginal distributions of $X$ and $Y$ are the same under
$P$ and $P_X \otimes P_Y$, it follows that
\begin{equation}
\begin{split}
\frac{d P_X}{d (P_X \otimes P_Y)_X} & = \frac{d P_X}{d P_X} = 1\quad\text{and}\\
\frac{d P_Y}{d (P_X \otimes P_Y)_Y} & = \frac{d P_Y}{d P_Y} = 1,
\end{split}
\end{equation}
and hence in the notation of Theorem~\ref{th:condDens} 
\begin{equation*}
g(y) = \int \varphi(x,y)\,P_X(dx) = 1 = 
\int \varphi(x,y)\,P_Y(dy) = h(x)
\end{equation*}
$P_Y$-a.s.\ for $y \in \Omega_Y$ and $P_X$-a.s.\ for $x \in \Omega_X$.
\end{proof}

Thanks to Proposition~\ref{pr:condDist}, under Assumption~\ref{as:aux} 
Fubini's theorem for conditional distributions \eqref{eq:Fubini} and
the formulae \eqref{eq:condExp} for conditional expected values under $P$ obtain simpler shapes. For instance,
the expected value of a $\mathcal{F}$-Borel measurable random variable 
$T: \Omega \to \mathbb{R}, (x,y) \mapsto T(x,y)$  with
$E_P[\lvert T(X,Y)\rvert] < \infty$ or $T \ge 0$ conditional on $X =x$ can be represented as
\begin{equation}\label{eq:consequence}
E_P[T(X,Y)\,|\,X=x] = \int T(x,y)\,\varphi(x,y) P_Y(dy),\quad P_X\text{-a.s.~for}\ x \in \Omega_X.
\end{equation}
\begin{subequations}
The same comment applies to the formulae from Corollary~\ref{co:genProperties}. The definitions of 
the conditional densities which are introduced in Theorem~\ref{th:condDens} remain unchanged under
Assumption~\ref{as:aux} such that for example there is no change in \eqref{eq:genDens} either. However, 
under Assumption~\ref{as:aux} the density $h$ of $Q_X$ with respect to $P_X$ as represented in 
\eqref{eq:hDens} may be calculated via 
\begin{equation}\label{eq:hDens2}
h(x) \ = \ \int q_{X|Y=y}(x)\,g(y)\,\varphi(x,y)\,P_Y(dy), \quad x \in \Omega_X.
\end{equation}
Under Assumption~\ref{as:aux}, moreover the posterior correction formula \eqref{eq:correction} 
can be rewritten (for $x \in \Omega_X$ with $h(x) >0$) as
\begin{equation}
Q[Y\in G\,|\,X=x] \ = \ \frac{\int_G q_{X|Y=y}(x)\,g(y)\,\varphi(x,y)\,P_Y(dy)}{h(x)},
\quad \ G \in \mathcal{G},
\end{equation}
with $h(x)$ as in \eqref{eq:hDens2}.
\end{subequations}

\section{Density estimation as binary classification problem}
\label{se:DensEst}

It is a well-known fact that a Radon-Nikodym density of a probability measure with respect to another
probability measure can be represented by means of the posterior probability in a binary classification
problem (Qin~\cite{Qin1998DensClass}; Sugiyama et al.~\cite{sugiyama2012density}, Chapter~4). 
With the following proposition, we provide a formal statement and proof of this fact.

\begin{proposition}\label{pr:DensEst}
Let $(\Mbar, \mathcal{M})$ be a measurable space and $\PP_0$ and $\PP_1$ be probability measures on 
$(\Mbar, \mathcal{M})$. Assume that $\lambda \ge 0$ is a density of $\PP_1$ with respect to $\PP_0$, i.e.\ 
it holds that $\lambda = \frac{d \PP_1}{d \PP_0}$.\\
Define $\Omegabar = \{0, 1\} \times \Mbar \times \Mbar$ and the $\sigma$-algebra $\Fbar$ on $\Omegabar$ by 
$\Fbar = \mathfrak{P}(\{0, 1\}) \otimes \mathcal{M} \otimes \mathcal{M}$ (the product $\sigma$-algebra of 
$\mathfrak{P}(\{0, 1\}) = \bigl\{\emptyset, \{0\}, \{1\}, \{0, 1\}\bigr\}$,
$\mathcal{M}$ and $\mathcal{M}$).

For $\omega = (s, m_0, m_1) \in \Omegabar$ denote by $S$, $X_0$ and $X_1$ the coordinate projections, i.e.\
$S(\omega) = s$, $X_0(\omega) = m_0$ and $X_1(\omega) = m_1$. \\
Let $0 < p < 1$ be fixed and define the probability measure $\PP_p$ on $(\Omegabar, \Fbar)$ as the product measure
such that $S$, $X_0$ and $X_1$ are independent under $\PP_p$ with $\PP_p[S=1] = p = 1 - \PP_p[S=0]$, 
$\PP_p[X_0 \in M] = \PP_0[M]$ and $\PP_p[X_1 \in M] = \PP_1[M]$ for $M \in \mathcal{M}$.\\
In addition, define $X: \Omegabar \to \Mbar$ by $X = \begin{cases} 
X_0, & \text{if}\ S=0\\
X_1, & \text{if}\ S=1
\end{cases}$\\
and $\eta_p(m) = \PP_p[S=1\,|\,X=m]$. Then it holds that
\begin{equation}\label{eq:posterior}
\eta_p(m) \ = \ \frac{p\,\lambda(m)}{p\,\lambda(m) + 1-p}, \quad \PP_0\text{-a.s.}\ m \in \Mbar.
\end{equation}
\end{proposition}

\begin{proof}[Proof of Proposition~\ref{pr:DensEst}.]
On the one hand, by the definition of conditional probability, it follows for $M \in \mathcal{M}$ that
\begin{align}
\PP_p[S=1, X \in M] & = E_{\PP_p}\bigl[\eta_p(X)\,\mathbf{1}_{\{X \in M\}}\bigr]\notag\\
& = E_{\PP_p}\bigl[\eta_p(X_0)\,\mathbf{1}_{\{X_0 \in M, S=0\}}\bigr] +
	E_{\PP_p}\bigl[\eta_p(X_1)\,\mathbf{1}_{\{X_0 \in M, S=1\}}\bigr]\notag\\
& = (1-p)\, E_{\PP_0}[\eta_p\,\mathbf{1}_M] + p\,E_{\PP_1}[\eta_p\,\mathbf{1}_M]\notag \\
& = E_{\PP_0}\bigl[\mathbf{1}_M\,\eta_p\,(p\,\lambda + 1-p) \bigr].\label{eq:second}
\end{align}
On the other hand, the definition of $\PP_p$ implies 
\begin{equation*}
\PP_p[S=1, X \in M] \ = \ \PP_p[S=1, X_1\in M]\ = \ p\,\PP_1[M] \ = \ p\,E_{\PP_0}[\lambda\,\mathbf{1}_M].
\end{equation*}
Together with \eqref{eq:second}, this observation implies $\eta_p(m)\,(p\,\lambda(m) + 1-p) = p\,\lambda(m)$ 
$\PP_0$-almost surely for $m \in \Mbar$. Since $p\,\lambda(m) + 1-p \ge 1-p > 0$,  
\eqref{eq:posterior} follows.
\end{proof}

Some consequences of Proposition~\ref{pr:DensEst} are as follows:
\begin{itemize}
\item \eqref{eq:posterior} implies $\eta_p(m) < 1$ for $m \in \Mbar$ $\PP_0$-a.s.\ and therefore also
\begin{equation}\label{eq:binaryDensity}
\lambda(m) \ = \ \frac{1-p}{p}\,\frac{\eta_p(m)}{1-\eta_p(m)}.
\end{equation}
\item \eqref{eq:binaryDensity}, in particular, implies that the parameter $p$ is eliminated when
\eqref{eq:posterior} is solved for $\lambda(m)$.
\item Suppose that there are two separate samples of realisations of a variable 
(with an arbitrary but fixed range of values),
and that the two samples have been generated under different distributions of the variable in question.
Assume further that one of these distributions is absolutely continuous with respect to the other distribution such
that there exists a density of the former distribution with respect to the latter distribution.
Proposition~\ref{pr:DensEst} can be used to estimate the density in the following way:
\begin{enumerate}
\item[1)] Combine the two samples and assign the label $0$ to the sample elements which come from the reference  
distribution for the density and the label $1$ to the sample elements which stem from the distribution 
which is absolutely continuous with respect to the other density.
\item[2)] Train a binary probabilistic binary classifier $\eta$ to classify the elements of the combined sample
according to their labels.
\item[3)] Let $p = \frac{\text{\# labels 1}}{\text{size of combined sample}}$.
\item[4)] Use \eqref{eq:binaryDensity} to represent the density in terms of the learnt probabilistic classifier $\eta$
and $p$ from step 3).
\end{enumerate}
\end{itemize}
\eqref{eq:binaryDensity} shows that over-sampling and under-sampling techniques may be deployed when training
the classifier without changing the resulting density.

Consider the case where the density $\varphi$ of $P=P_{(X,Y)}$ with respect to $P_X \otimes P_Y$ according to
Assumption~\ref{as:aux} is estimated by means of binary classification as described above.
Then one sample from
$P$ suffices because a sample for $P_X \otimes P_Y$ can be generated as the set of pairs of each element of the 
marginal sample for $X$ and each element of the marginal sample for $Y$. But this method creates strong 
imbalance of the
labels in the classification sample. Therefore, the robustness of the classification approach with respect to under- 
and over-sampling becomes important.

\section{Proofs}
\label{se:proofs}

\begin{proof}[Proof of Proposition~\ref{pr:SpecCases}.]
Denote by $f$ the density of $Q$ with respect to $P$ that exists by Assumption~\ref{as:main}.
Define the $\mathcal{H}$-measurable function $h \ge 0$ by $h(x) = E_P[f\,|\,X=x]$ for $x \in \Omega_X$.
Then for $H \in \mathcal{H}$ and $G \in \mathcal{G}$, under covariate shift it follows that
\begin{align*}
Q[X\in H, Y \in G] & = E_Q[\mathbf{1}_{\{X\in H\}}\,\eta_G(X)]\\
	& = E_P[f\,\mathbf{1}_{\{X\in H\}}\,\eta_G(X)]\\
	& = E_P\bigl[E_P[f\,|\,\sigma(X)]\,\mathbf{1}_{\{X\in H\}}\,P[Y\in G\,|\,\sigma(X)\bigr]\\
	& = E_P[h(X)\,\mathbf{1}_{\{X\in H, Y\in G\}}].
\end{align*}
By \eqref{eq:F}, $\bigl\{\{X\in H, Y \in G\}: H \in \mathcal{H}, G \in \mathcal{G}\bigr\}$ is
a $\cap$-stable generator of $\mathcal{F}$, hence by the measure uniqueness theorem  
(Bauer~\cite{Bauer1981ProbTheory}, Theorem~1.5.5) $h(X)$ is
a density of $Q$ with respect to $P$ on $(\Omega, \mathcal{F})$. Thus, $P$ and $Q$ are related through 
FJS with $\hbar = h$ and $\gbar = 1$. Thanks to the symmetry of the setting with respect to 
$X$ and $Y$, we also have shown that label shift implies FJS with $\hbar = 1$ and $\gbar = E_P[f\,|\,Y=\cdot]$.\\
The converse claims follow from Lemma~4.2 of Tasche~\cite{tasche2023sparse}.
\end{proof}

\begin{proof}[Proof of Proposition~\ref{pr:combined}.]\
\begin{itemize}
\item[(i)]  $E_P\bigl[\gbar(Y)\bigr] =0$ would
imply $\gbar(Y) = 0$ $P$-a.s.~and hence also $E_P\bigl[\hbar(X)\,\gbar(Y)\bigr] = 0$. This would contradict
$E_P\bigl[\hbar(X)\,\gbar(Y)\bigr] = 1$ which is true by definition of FJS. Thus $E_P\bigl[\gbar(Y)\bigr] >0$ is true.
From the definition of $Q_L$ it follows that 
\begin{equation*}
E_{Q_L}\bigl[\hbar(X)\bigr]\ =\ \frac{E_P\bigl[\hbar(X)\,\gbar(Y)\bigr]}{E_P\bigl[\gbar(Y)\bigr]} \ = \
\frac{1}{E_P\bigl[\gbar(Y)\bigr]} \ > \ 0.
\end{equation*}
Hence also $Q_C$ is well-defined.
\item[(ii)] For all $F\in \mathcal{F}$ it follows from the definitions of $Q_C$ and $Q_L$ that
\begin{equation*}
Q_C[F] \ = \ \frac{E_{Q_L}\bigl[\hbar(X)\,\mathbf{1}_F\bigr]}{E_{Q_L}\bigl[\hbar(X)\bigr]}\ = \
\,E_P\bigl[\hbar(X)\,\gbar(Y)\,\mathbf{1}_F\bigr]\frac{E_P\bigl[\gbar(Y)\bigr]}{E_P\bigl[\gbar(Y)\bigr]}\ =\ Q[F].
\end{equation*}
\item[(iii)] Both equations in (iii) follow from Proposition~\ref{pr:SpecCases}. 
\end{itemize}
\end{proof}

\begin{proof}[Proof of Theorem~\ref{th:alternateB}.]\ 
\begin{description}
\item[ad (i)] $0 < g_n(Y) < \infty$ implies $0 < E_P[g_n(Y)\,|\,\sigma(X)] = h_n(X) < \infty$. \\
$E_P[h_n(X)] = 1$ follows from $1 = E_P[g_n(Y)] = E_P\bigl[E_P[g_n(Y)\,|\,\sigma(X)]\bigr]$.\\
By \eqref{eq:gnB}, $0 < g_{n+1}(Y) < \infty$ follows from the corresponding properties of
$h(X)$, $g_n(Y)$ and $h_n(X)$. 
$g_{n+1}(Y)$ is a density because 
\begin{equation*}
E_P[g_{n+1}(Y)] = E_P\left[\frac{h(X)}{h_n(X)}\,E_P[g_n(Y)\,|\,\sigma(X)]\right]
= E_P[h(X)] =1.
\end{equation*}

\item[ad (ii)] $E_P[f_n] = E_P\left[\frac{h(X)}{h_n(X)}\,g_n(Y)\right] = 
E_P\left[\frac{h(X)}{h_n(X)}\,E_P[g_n(Y)\,|\,\sigma(X)]\right]$\\ 
$= E_P[h(X)] =1$.

\item[ad (iii)] $\frac{d Q_Y^{(n)}}{d P_Y}(y) = E_P[g_n(Y)\,|\,Y=y] = g_n(y)$,\\ 
$\frac{d Q_X^{(n)}}{d P_X}(x) = E_P[g_n(Y)\,|\,X=x] = h_n(x)$,\\ 
$\frac{d R_X^{(n)}}{d P_X}(x) = E_P\left[\frac{h(X)}{h_n(X)}\,g_n(Y)\,\Big|\,X=x\right]
= \frac{h(x)}{h_n(x)}\,E_P[g_n(Y)\,|\,X=x] = h(x)$,\\ 
$\frac{d R_Y^{(n)}}{d P_Y}(y) = E_P\left[\frac{h(X)}{h_n(X)}\,g_n(Y)\,\Big|\,Y=y\right] =
g_n(y)\,E_P\left[\frac{h(X)}{h_n(X)}\,\Big|\,Y=y\right] = g_{n+1}(y)$.

\item[ad (iv)] Observe that 
\begin{equation*}
E_P[f_n\,\log(g_k(Y))] = E_P\bigl[E_P[f_n\,|\,\sigma(Y)]\,\log(g_k(Y))\bigr]
= E_P[g_{n+1}(Y)\,\log(g_k(Y))]
\end{equation*}
\begin{subequations}
for $k \in\{n, n+1\}$ by \eqref{eq:gnB}. Jensen's inequality implies
\begin{equation}\label{eq:Jens1}
E_P[g_{n+1}(Y)\,\log(g_{n+1}(Y))] \ \ge\ E_P[g_{n+1}(Y)\,\log(g_n(Y))].
\end{equation}
From the assumptions on the integrability of $\log(h_n(X))$ and $\log(g_k(Y))$ it follows
that $E_P[f_n\,\lvert\log(f_k)\rvert] < \infty$ for $k \in\{n, n+1\}$. Again Jensen's inequality implies
\begin{align}
 E_P[f_n\,\log(f_{n+1})] & \ \le \ E_P[f_n\,\log(f_n)]\notag\\
 \iff  E_P\left[f_n\,\log\left(\frac{g_{n+1}(Y)}{h_{n+1}(X)}\right)\right] &\ \le \ 
 		E_P\left[f_n\,\log\left(\frac{g_n(Y)}{h_n(X)}\right)\right]. \label{eq:Jens2}
\end{align}
Now (iv) follows from the calculation
\begin{align*}
E_P[h(X)\,\log(h_{n+1}(X))] & = E_P\bigl[E_P[f_n\,|\,\sigma(X)]\,\log(h_{n+1}(X))\bigr]\\
& = E_P[f_n\,\log(h_{n+1}(X))] \\
& = E_P[g_{n+1}(Y)\,\log(g_{n+1}(Y))] - E_P\left[f_n\,\log\left(\frac{g_{n+1}(Y)}{h_{n+1}(X)}\right)\right]\\
& \ge E_P[g_{n+1}(Y)\,\log(g_n(Y))] - E_P\left[f_n\,\log\left(\frac{g_n(Y)}{h_n(X)}\right)\right]\\
& = E_P[f_n\,\log(h_n(X))] \\
& = E_P[h(X)\,\log(h_n(X))],
\end{align*}
where the inequality is implied by \eqref{eq:Jens1} and \eqref{eq:Jens2}.
\end{subequations}
\end{description}
\end{proof}

\begin{proof}[Proof of Proposition~\ref{pr:convergence}.]
Since $\int g\, dP_Y =1 = 
\int g_n\, dP_Y$ for all $n$ and $g_n(Y) \to g(Y)$ in probability, the sequence $g_n(Y)$, $n = 0, 1, \ldots$, 
is uniformly integrable under $P$  and
it follows that $\lim_{n\to\infty} E_P\bigl[\lvert g_n(Y) - g(Y)\rvert \bigr]  = 0$, i.e.~(i) 
(Ethier and Kurtz~\cite{ethier1986markov}, Proposition~2.3 of Appendix, and Klenke~\cite{klenke2013probability},
Theorem~6.25).
This observation implies that  $h_n(X) = E_P[g_n(Y)\,|\,\sigma(X)]$, $n = 0, 1, \ldots$, converges in
mean to $\hhat(X) = E_P[g(Y)\,|\,\sigma(X)] > 0$ (Klenke~\cite{klenke2013probability}, Corollary~8.20). 
This proves (ii).
(iii) follows from (ii) and Corollary~\ref{co:KL} thanks to Fatou's lemma. \\
Thanks to $g_n \to g >0$, \eqref{eq:gnB} implies  $E_P\left[\frac{h(X)}{h_n(X)}\,\Big|\,\sigma(Y)\right] 
	\xrightarrow[n\to\infty]{} 1$ in probability under $P$.
By Theorem~20.5 of Billingsley~\cite{billingsley1986probability} and (ii) there is a sub-sequence 
$n(1) < n(2) < \ldots$ of $0, 1, 2, \ldots$ such
that both $\lim_{k\to\infty} E_P\left[\frac{h(X)}{h_{n(k)}(X)}\,\Big|\,\sigma(Y)\right]   = 1$ 
$P$-a.s.\  and
$\lim_{k\to\infty} h_{n(k)}(X) = \hhat(X)$ $P$-a.s.\ hold true.
Hence Fatou's lemma  for conditional expectations implies
\begin{align*}
1 & = \lim\limits_{k\to\infty} E_P\left[\frac{h(X)}{h_{n(k)}(X)}\,\Big|\,\sigma(Y)\right] \\
& \ge E_P\left[\lim\limits_{k\to\infty} \frac{h(X)}{h_{n(k)}(X)}\,\Big|\,\sigma(Y)\right]\\
& = E_P\left[\frac{h(X)}{\hhat(X)}\,\Big|\,\sigma(Y)\right]. 
\end{align*}
This completes the proof.
\end{proof}

\begin{proof}[Proof of Proposition~\ref{pr:alternate}.] The proof is carried out with induction over $n$.
For any $n$, $0 < \gbar_{n}(Y) < \infty$ $P$-a.s.\ implies $0 < E_P[\gbar_{n}(Y)\,|\,\sigma(X)] < \infty$ $P$-a.s.
By \eqref{eq:hn}, it follows that $0 < \hbar_n(X) < \infty$ $P$-a.s.
Again by \eqref{eq:hn} we obtain for all $H\in\mathcal{H}$
\begin{align*}
E_P[\mathbf{1}_H(X)\,h(X)] & \ = \ E_P\bigl[\mathbf{1}_H(X)\,\hbar_n(X)\,E_P[\gbar_n(Y)\,|\,\sigma(X)]\bigr] \\
& \ = \ E_P[\mathbf{1}_H(X)\,\hbar_n(X)\,\gbar_n(Y)].
\end{align*}
Since $h$ is a density under $P_X$, this proves that $\hbar_n(X)\,\gbar_n(Y)$ is a density under $P$ 
as well as $\frac{d Q_X^{(n)}}{d P_X} = h$.\\ 
Now $0 < \hbar_n(X) < \infty$ $P$-a.s.\ implies
$0 < E_P[\hbar_{n}(X)\,|\,\sigma(Y)] < \infty$ $P$-a.s.\ and 
$0 < \gbar_{n+1}(Y) < \infty$ $P$-a.s.\ by \eqref{eq:gn}. 
Eq.~\eqref{eq:gn} also implies for all $G\in \mathcal{G}$
\begin{align*}
E_P[\mathbf{1}_G(Y)\,g(Y)] &\ = \ E_P\bigl[\mathbf{1}_G(Y)\,\gbar_{n+1}(Y)\,E_P[\hbar_n(X)\,|\,\sigma(Y)]\bigr] \\
&\ = \ E_P[\mathbf{1}_G(Y)\,\hbar_n(X)\,\gbar_{n+1}n(Y)].
\end{align*}
Since $g$ is a density under $P_Y$, this proves that $\gbar_{n+1}(Y)\,\hbar_n(X)$ is a density under $P$ 
as well as $\frac{d R_Y^{(n)}}{d P_Y} = g$.
\end{proof}

\begin{proof}[Proof of Proposition~\ref{pr:GLS}.]
Observe that 
\begin{equation*}
Q[\hbar(R)\,\gbar(Y) > 0] = 1 = Q\bigl[\hbar(R)\,E_P[\gbar(Y)\,|\,\sigma(R)] > 0\bigr]
\end{equation*}
(proof like in Lemma~A1 of Tasche~\cite{tasche2022factorizable}). This implies 
\begin{subequations}
\begin{equation}\label{eq:positive}
Q[\hbar(R) >0]\ =\ 1\ = \ Q\bigl[E_P[\gbar(Y)\,|\,\sigma(R)] > 0\bigr].
\end{equation}
Thanks to the existence of $P_{Y|R}$, \eqref{eq:sufficiency} implies
$E_P[T(Y)\,|\,\sigma(R)] = E_P[T(Y)\,|\,\sigma(X)]$ for all $\mathcal{G}$-Borel measurable functions
$T \ge 0$. Therefore, it follows from \eqref{eq:positive} that
\begin{equation}\label{eq:denom}
0\ = \ Q\bigl[E_P[\gbar(Y)\,|\,\sigma(X)] = 0\bigr] \ = \ 
E_P\bigl[E_P[f(X,Y)\,|\,\sigma(X)]\,\mathbf{1}_{\{E_P[\gbar(Y)\,|\,\sigma(X)]=0\}}\bigr].
\end{equation}
\end{subequations}
\eqref{eq:positive} and \eqref{eq:denom} together with the sufficiency assumptions and the generalized Bayes formula
imply that the following calculation is valid for 
all $H \in \mathcal{H}$ and $G \in \mathcal{G}$:
\begin{align*}
Q[X\in H, Y \in G] & = Q\bigl[\mathbf{1}_{\{X\in H\}}\,Q[Y\in G\,|\,\sigma(X)]\bigr]\\
& = Q\bigl[\mathbf{1}_{\{X\in H\}}\,Q[Y\in G\,|\,\sigma(R)]\bigr]\\
& = E_Q\left[\mathbf{1}_{\{X\in H\}}\,\frac{E_P[\gbar(Y)\,\mathbf{1}_{\{Y\in G\}}\,|\,\sigma(R)]}
	{E_P[\gbar(Y)\,|\,\sigma(R)]} \right]\\
& = E_P\left[f(X,Y)\,\mathbf{1}_{\{X\in H\}}\,\frac{E_P[\gbar(Y)\,\mathbf{1}_{\{Y\in G\}}\,|\,\sigma(R)]}
	{E_P[\gbar(Y)\,|\,\sigma(R)]} \right]\\
& = E_P\left[E_P[f(X,Y)\,|\,\sigma(X)]\,\mathbf{1}_{\{X\in H\}}\,
	\frac{E_P[\gbar(Y)\,\mathbf{1}_{\{Y\in G\}}\,|\,\sigma(R)]}	{E_P[\gbar(Y)\,|\,\sigma(R)]} \right]\\
& = E_P\left[\mathbf{1}_{\{X\in H\}\cap \{Y\in G\}}\, 
	\frac{E_P[f(X,Y)\,|\,\sigma(X)]}{E_P[\gbar(Y)\,|\,\sigma(X)]}\,\gbar(Y)\right].
\end{align*}
This implies \eqref{eq:alsoGLS}.
\end{proof}

\begin{proof}[Proof of Theorem~\ref{th:condDens}.]
Observe that $(x,y) \mapsto q_{Y|X=x}(y)$ is $\mathcal{H}\otimes\mathcal{G}$-Borel
measurable. Define $K: \Omega_X \times \mathcal{G} \to [0,\infty)$ by 
\begin{equation*}
K(x, G)\  =\ \begin{cases}  
\int_G q_{Y|X=x}(y)\,P_{Y|X=x}(dy), & h(x) >0 \\
\mathbf{1}_{y_0}(G), & h(x)=0
\end{cases}
\end{equation*}
with a fixed arbitrary $y_0 \in \Omega_Y$,  for $x \in \Omega_X$ and $G\in\mathcal{G}$. 
Then $K$ satisfies (i) of Definition~\ref{de:regCondDist} and $K(x, \cdot)$ is a measure on $(\Omega_Y, \mathcal{G})$. 
By Fubini's theorem \eqref{eq:Fubini}, it follows for $H \in \mathcal{H}$ and $G \in \mathcal{G}$  that
\begin{align*}
Q[X \in H, Y \in G] & = \int_{H \cap \{h > 0\}} \left(\int_G q_{Y|X=x}(y)\,P_{Y|X=x}(dy)\right) h(x)\,P_X(dx)\\
& = \int_{H \cap \{h > 0\}} K(x,G)\,Q_X(dx)  \\
& = \int_H K(x,G)\,Q_X(dx).
\end{align*}
By \eqref{eq:condV}, this observation implies (iii) of Definition~\ref{de:regCondDist} 
for every $G \in \mathcal{G}$. That $K(x,\cdot)$ is a probability measure follows from (iii) of
Definition~\ref{de:regCondDist} for $G=\Omega_Y$. Hence $Q_{Y|X} = K$ is a regular conditional 
distribution of $Y$ with respect to $X$ under $Q$. The density property 
of $q_{Y|X=x}$ for $x$ with $h(x) >0$ then follows from the definition of $K$. This implies (i) of 
Theorem~\ref{th:condDens}. (ii) follows from (i) since the roles of $X$ and $Y$ may be swapped thanks
to the symmetry of the setting in Section~\ref{se:setting}.
\end{proof}


\providecommand{\href}[2]{#2}
\providecommand{\arxiv}[1]{\href{http://arxiv.org/abs/#1}{arXiv:#1}}
\providecommand{\url}[1]{\texttt{#1}}
\providecommand{\urlprefix}{URL }

\end{document}